\documentclass{article}

\usepackage{microtype}
\usepackage{graphicx}
\usepackage{subfigure}
\usepackage{booktabs} 
\usepackage{hyperref}

\usepackage[accepted]{icml2024}

\usepackage{amsmath}
\usepackage{amssymb}
\usepackage{mathtools}
\usepackage{amsthm}
\usepackage{makecell}
\usepackage{array}
\newcolumntype{C}[1]{>{\centering\arraybackslash}m{#1}}

\usepackage[capitalize,noabbrev]{cleveref}

\theoremstyle{plain}
\newtheorem{theorem}{Theorem}[section]
\newtheorem{proposition}[theorem]{Proposition}
\newtheorem{lemma}[theorem]{Lemma}
\newtheorem{corollary}[theorem]{Corollary}
\theoremstyle{definition}
\newtheorem{definition}[theorem]{Definition}
\newtheorem{assumption}[theorem]{Assumption}
\theoremstyle{remark}

\newcommand{\norm}[1]{\left\lVert#1\right\rVert}
\newcommand{\expc}[1]{\mathsf{E}\!\left[#1\right]}

\usepackage[textsize=tiny]{todonotes}

\icmltitlerunning{FedSC: Provable Federated Self-supervised Learning with Spectral Contrastive Objective over Non-i.i.d. Data}

\begin{document}

\twocolumn[
\icmltitle{FedSC: Provable Federated Self-supervised Learning  \\ with Spectral Contrastive Objective over Non-i.i.d. Data}



\icmlsetsymbol{equal}{*}

\begin{icmlauthorlist}
\icmlauthor{Shusen Jing}{UCSF}
\icmlauthor{Anlan Yu}{Lehigh}
\icmlauthor{Shuai Zhang}{NJIT}
\icmlauthor{Songyang Zhang}{LOU}
\end{icmlauthorlist}

\icmlaffiliation{UCSF}{Department of Radiation Oncology, University of California, San Francisco, California, USA}

\icmlaffiliation{Lehigh}{Department of Electrical and Computer Engineering, Lehigh University, Bethlehem, Pennsylvania, USA}

\icmlaffiliation{NJIT}{Department of Data Science, New Jersey Institute of Technology, Newark, New Jersey, USA}

\icmlaffiliation{LOU}{Department of Electrical and Computer Engineering, University of Louisiana at Lafayette, Lafayette, Louisiana, USA}

\icmlcorrespondingauthor{Shusen Jing}{shusen.jing@ucsf.edu}
\icmlcorrespondingauthor{Songyang Zhang}{songyang.zhang@louisiana.edu}

\icmlkeywords{Machine Learning, ICML}

\vskip 0.3in
]



\printAffiliationsAndNotice{}  

\begin{abstract}
Recent efforts have been made to integrate self-supervised learning (SSL) with the framework of federated learning (FL). One unique challenge of federated self-supervised learning (FedSSL) is that the global objective of FedSSL usually does not equal the weighted sum of local SSL objectives. Consequently, conventional approaches, such as federated averaging (FedAvg), fail to precisely minimize the FedSSL global objective, often resulting in suboptimal performance, especially when data is non-i.i.d.. To fill this gap, we propose a provable FedSSL algorithm, named FedSC, based on the spectral contrastive objective. In FedSC, clients share correlation matrices of data representations in addition to model weights periodically, which enables inter-client contrast of data samples in addition to intra-client contrast and contraction, resulting in improved quality of data representations. Differential privacy (DP) protection is deployed to control the additional privacy leakage on local datasets when correlation matrices are shared. We also provide theoretical analysis on the convergence and extra privacy leakage. The experimental results validate the effectiveness of our proposed algorithm.
\end{abstract}

\section{Introduction}


As a type of unsupervised learning, self-supervised learning (SSL) aims to learn a structured representation space, in which data similarity can be measured by simple metrics, such as cosine and Euclidean distances, with unlabeled data \cite{chen2020simple,chen2021exploring,grill2020bootstrap,he2020momentum,zbontar2021barlow,bardes2021vicreg,haochen2021provable}. On top of the foundation model trained with SSL, a simple linear layer, also known as linear probe, is sufficient to perform well on a wide range of downstream tasks with minimal labeled data. Resulting from its high label efficiency, SSL has been adopted in a variety of applications, such as natural language processing \cite{he2021towards, brown2020language} and computer vision \cite{ravi2016optimization,hu2021lora}.


However, SSL algorithms are often executed on massive amounts of unlabeled data that may be dispersed across various locations. Moreover, the progressively tightening privacy-protection regulations frequently inhibit the centralization of data. Within this context, the federated learning (FL) framework is often favored, wherein a central server can learn from private data located on clients without the data being shared directly \cite{mcmahan17a, stich2018local,li2019convergence}.


Despite the extensive study and theoretical guarantees \cite{stich2018local,li2019convergence} associated with conventional FL, its generalization to incorporate with SSL is not straightforward. 
The \textit{fundamental challenge} arises from the fact that, unlike FL within supervised learning, the global objective of FedSSL usually does not equal the weighted sum of local SSL objectives. Consequently, conventional FL approaches, e.g. federated averaging (FedAvg), can not minimize the exact global objective of FedSSL especially when data is non-independent and identically distributed (non-i.i.d.). From the perspective of contrastive learning, FedAvg only contrasts data samples within the same client (intra-client) rather than those across different clients (inter-client). 
Therefore, the learned representation might not be as effective at distinguishing inter-client data samples as it is with intra-client data samples.

Although recent works on FedSSL have shown great numerical success \cite{zhuang2021collaborative,zhuang2022divergence,zhang2023federated,han2022fedx}, the majority of them either overlook previously mentioned challenge or fail to offer a theoretical analysis. FedU \cite{zhuang2021collaborative} and FedEMA \cite{zhuang2021collaborative} lack the formulation of global objective and thus fail to provide theoretical analysis. FedCA \cite{zhang2023federated} notices the unique challenge and proposes to share data representations, which, however, results in significant privacy leakage and communication overhead. Unlike FedU and FedEMA, which involve sharing predictors, and FedCA, which shares data representations, our proposed FedSC results in much lower communication costs, since sharing correlation matrices requires transmitting far fewer parameters than what is needed for predictors or data representations. FedX \cite{han2022fedx} does not share additional information besides encoders, but still lacks theoretical analysis. Among all these works, only our proposed FedSC deploys differential privacy (DP) protection to mitigate the extra privacy leakage from components other than encoders. Moreover, FedSC is the only provable FedSSL method to the best knowledge of the authors. Table \ref{tab:comp0} summarizes the difference between this work and state of the arts (SOTAs).

\begin{table}[t]
\vspace{-0.5 em}
\caption{A comparison with SOTAs: FedSC (proposed) is the only one applying DP mechanism on components other than encoder. Moreover, FedSC is the only provable method among them.}
\begin{center}
\begin{footnotesize}
\begin{tabular}{c|c|c|c}
\hline
          & \makecell{Info. shared \\ besides encoder}  &  \makecell{Privacy \\ Protection}    &  Provable  \\
\hline
FedU  & predictor & $ \times$ & $\times $\\
\hline
FedEMA & predictor & $ \times$ & $ \times$\\
\hline
FedX  & N/A  & $\times$ & $\times$ \\
\hline
FedCA  & representations & $\times $ & $\times $ \\
\hline
FedSC & correlation matrices & $\surd$ & $\surd$ \\ 
\hline
\end{tabular}
\end{footnotesize}
\end{center}
\label{tab:comp0}
\vspace{-1.5 em}
\end{table}


\textbf{Contribution.} In this work, we propose a novel FedSSL formulation based on the spectral contrastive (SC) objective \cite{haochen2021provable}. The formulation clarifies all the necessary components in FedSSL encompassing intra-client contraction, intra-client contrast and inter-client contrast. Building upon this formulation, we propose the first provable FedSSL method, namely \textbf{FedSC}, with the convergence guarantee to the solutions of centralized SSL. Unlike FedAvg, clients in FedSC share correlation matrices of their local data representations in addition to the weights of local models. By leveraging the aggregated correlation matrix from the server, inter-client contrast of data samples, which is overlooked in FedAvg, can be performed in addition to local contrast and contraction. To better control and quantify the extra privacy leakage, we apply DP mechanism to correlation matrices when they are shared. We made theoretical analysis of FedSC, demonstrating the convergence of the global objective and efficacy of our method. Our contributions are summarized as follows:


    $\bullet$ We propose a novel FedSSL formulation delineating all essential components of FedSSL, which encompasses intra-client contraction, intra-client contrast and inter-client contrast. This highlights the limitations of FedAvg due to its neglect of the inter-client contrast.

    $\bullet$ We propose FedSC, in which clients are able to perform inter-client contrast of data samples by leveraging the correlation matrices of data representations shared from others, resulting in improved quality of data representations.
    
    $\bullet$ DP protection is applied, which effectively constrains the privacy leakage resulting from sharing correlation matrices with only negligible utility degradation.
    
    $\bullet$ Theoretical analysis of FedSC is made, providing extra privacy leakage and convergence guarantee for the global FedSSL objective. We prove that FedSC can achieve a $\mathcal{O}(1/\sqrt{T})$ convergence rate, while FedAvg will have a constant error floor.
    
    $\bullet$ Through extensive experimentation involving $3$ datasets across $4$ SOTAs, we affirm that FedSC achieves superior or comparable performance compared with other methods. 

\section{Related Works}
\textbf{Self-supervised learning.} SSL can be mainly categorized into contrastive and non-contrastive SSL. The mechanisms and explicit objective of non-contrastive  SSL algorithms are still not fully understood despite a few recent attempts \cite{halvagal2023implicit,tian2021understanding,zhang2022does}. In contrast, contrastive SSL is more intuitive and explainable. Contrastive SSL explicitly penalizes the distance between positive pairs (two samples share the same semantic meaning), while encouraging distance between negative pairs (two samples share different semantic meanings). For example, SimCLR \cite{chen2020simple} objective accounts for the mutual information between positive pairs \cite{tschannen2019mutual} preserved by representations. The SC objective \cite{haochen2021provable} is equivalent to performing a spectral decomposition of the augmentation graph.

\textbf{Federated Self-supervised Learning.}
In FedU \cite{zhuang2021collaborative}, clients make decisions on whether the local model should be updated by the global based on the distances of two model weights when receiving global models from the server. As a follow up, FedEMA \cite{zhuang2022divergence} is proposed, in which the hard decision in FedU is replaced with a weighted combination of local and global models. FedX \cite{han2022fedx} designs local and global objectives using the idea of cross knowledge distillation to mitigate the effects of non-i.i.d. data. 
The authors of FedCA \cite{zhang2023federated} propose to share features of individual data samples in addition to local model weights for inter-client contrast, which however, results in significant privacy leakage and communication overhead.

\textbf{Differential Privacy.} Gaussian and Laplace mechanisms are most common DP approaches to protect a dataset from membership attack \cite{dwork2006differential}. To better analyze DP, \cite{mironov2017renyi} proposes R{\'e}nyi differential privacy (RDP), which characterizes the operations on mechanisms, such as composition, in a more elegant way, and proves the equivalence between DP and RDP. Currently, DP has been widely deployed in FL \cite{wei2020federated,truex2020ldp,hu2020personalized,geyer2017differentially,noble2022differentially}.


\begin{figure*}[htpt]
\centering
\includegraphics[width=0.61 \linewidth]{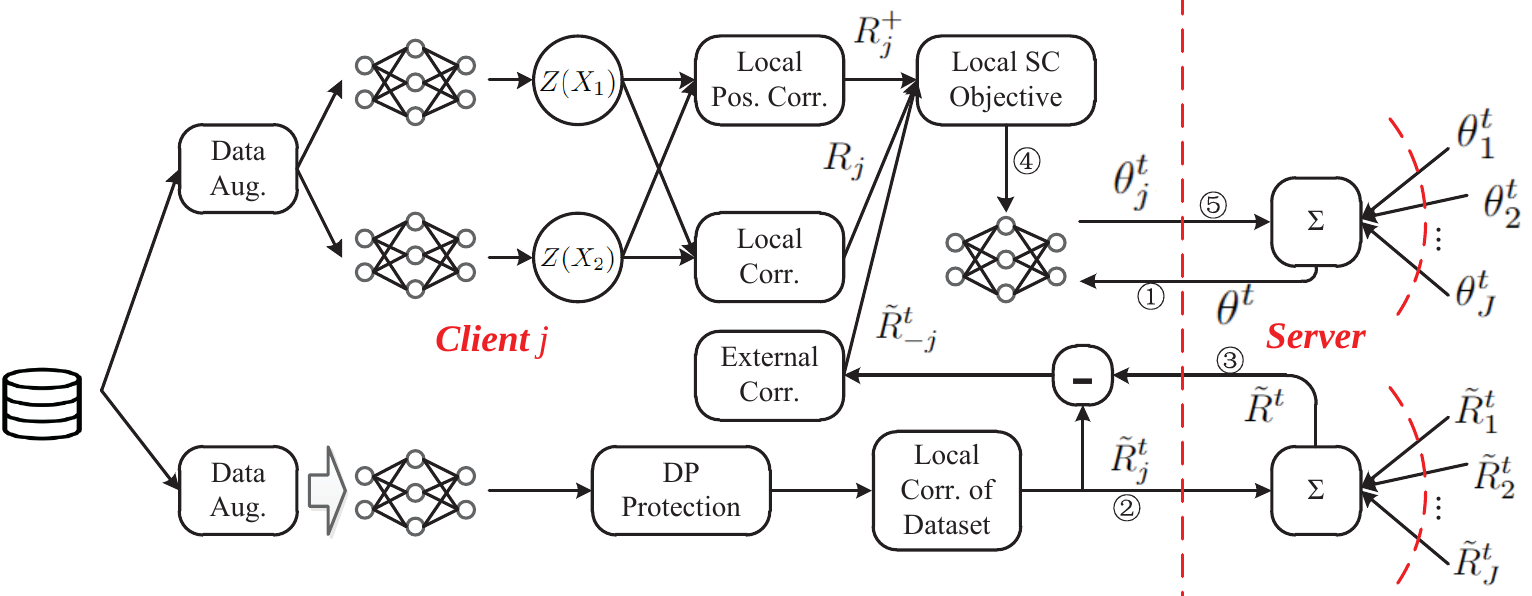}
\caption{Diagram of the proposed FedSC. 1) The server synchronizes local models with the global model. 2) Clients compute their local correlation matrices of dataset and send them to the server. 3) The server distributes the aggregated global correlation matrices back to the clients. 4) The clients proceed to update their local models in accordance with the local objective specified in Eq.  (\ref{eq:local}). 5) The server aggregates the local models and initiates the next iteration.} 
\label{fig:FLarc}
\vspace{-1 em}
\end{figure*}

\section{Preliminaries: Spectral Contrastive (SC) Self-supervised Learning}

Spectral contrastive (SC) SSL is proposed in \cite{haochen2021provable} with the following objective:
\begin{gather}
\begin{aligned}
&\mathcal{L}^{SC}(\theta; \mathcal{D}) \triangleq  \frac{1}{2}\mathsf{E}_{x,x^-\sim \mathcal{A}(\cdot|\mathcal{D})}\left[\left(z(x;\theta)^Tz(x^-;\theta)\right)^2\right] \\
& -\mathsf{E}_{\Bar{x}\sim \mathcal{D}}\mathsf{E}_{x,x^+\sim \mathcal{A}(\cdot|\Bar{x})}\left[z(x;\theta)^Tz(x^+;\theta)\right] \nonumber,\\
\end{aligned}
\end{gather}
where $\mathcal{D}$ is the dataset; $z(\cdot;\theta): \mathbb{R}^d \rightarrow \mathbb{R}^{H}$ is the representation mapping parameterized by $\theta$; $\mathsf{E}$ is referred to as the operator of expectation; $\mathcal{A}(\cdot|\Bar{x})$ is referred to as the augmentation kernel, which is essentially a conditional distribution, and $\mathcal{A}(\cdot|\mathcal{D}) \triangleq \mathsf{E}_{\Bar{x}\sim \mathcal{D}}\mathcal{A}(\cdot|\Bar{x})$. We use $(x, x^-)$ to denote negative pairs, where sample $x$ and $x^-$ have different semantic meanings, and $(x, x^+)$ to denote positive pairs, where $x$ and $x^+$ have same semantic meaning. 
Intuitively, minimizing the SC objective encourages the orthogonality of representations of a negative pair, and simultaneously promotes linear alignment of representations of a positive pair. 
It has been proved that solving this optimization problem is equivalent to doing spectral decomposition of a well-defined augmentation graph, whose nodes are augmented images, i.e, from $\mathcal{A}(\cdot|\mathcal{D})$, and edges describe the semantic similarity of two images determined by the kernel $\mathcal{A}(\cdot|\cdot)$, which results in high-quality and explainable data (node) representations \cite{haochen2021provable}.


After rearranging the original SC objective, we first propose the following equivalent  form not reported in \cite{haochen2021provable}.
\begin{gather}\label{eq:SC}
\begin{aligned}
\mathcal{L}^{SC}(\theta; \mathcal{D}) &= -\mathsf{E}_{\Bar{x}\sim \mathcal{D}} Tr\{R^+(\Bar{x};\theta)\} \\
& + \frac{1}{2} \norm{\mathsf{E}_{\Bar{x}\sim \mathcal{D}}R(\Bar{x};\theta)}_F^2    ,
\end{aligned}
\end{gather}
where $R^+(\Bar{x};\theta)\in\mathbb{R}^{H\times H}$ and $R(\Bar{x};\theta)\in\mathbb{R}^{H\times H}$ are defined respectively as
\begin{gather}
\begin{aligned}
R^+(\Bar{x};\theta) &\triangleq \mathsf{E}_{x,x^+\sim \mathcal{A}(\cdot|\Bar{x})}\left[z(x;\theta)z(x^+;\theta)^T\right], \\
R(\Bar{x};\theta) &\triangleq \mathsf{E}_{x\sim \mathcal{A}(\cdot|\Bar{x})}\left[z(x;\theta)z(x;\theta)^T\right].
\end{aligned}\nonumber
\end{gather}
The detailed derivations are given in Appendix \ref{sec: SC_derivative}. 

\section{Problem Formulation}
In an FL system consists of a server and $J$ clients, the $j$-th client owns a private local dataset $\mathcal{D}_j$ disjoint with others. The goal of FedSSL is to optimize the SSL model (SC model in this work) over the union of all local datasets, 
i.e, 
\begin{equation}
\vspace{-1mm}
    \min_{\theta}\mathcal{L}^{SC}(\theta; \mathcal{D}),
\end{equation}
where $\mathcal{D} = \cup_{j=1}^J\mathcal{D}_j$. Like the majority of SSL objectives, the global SC objective typically does not equal the weighted sum of local SC objectives, especially with non-i.i.d. data distribution. For the purpose of rigor, we make it an assumption instead of a claim in this work as follows.
\begin{gather}\label{eq:neq}
\vspace{-1mm}
\mathcal{L}^{SC}(\theta; \mathcal{D})\neq \sum_{j=1}^Jq_j  \mathcal{L}^{SC}(\theta; \mathcal{D}_j),
\end{gather}
where $\{q_j\}$ are weights depending on the amount of local data. As a result, FedAvg is not guaranteed to minimize the global objective $\mathcal{L}^{SC}(\theta; \mathcal{D})$ when data is non-i.i.d.. 

In addition, we adopt SC framework for the following reasons: First, SC has solid theoretical derivations and simultaneously achieves performance comparable to SOTA SSL methods  \cite{haochen2021provable}. Second, the SC objective suggests that correlation matrices of data representations are sufficient for contrasting negative-pairs. Sharing correlation matrices only results in constant negligible extra communication overheads and quantifiable privacy leakage.


\section{FedSC: A Provable FedSSL Method}
For the simplification of notations, we denote $R^+_j(\theta) = \mathsf{E}_{\Bar{x}\sim \mathcal{D}_j} R^+(\Bar{x};\theta)$ and $R_j(\theta) = \mathsf{E}_{\Bar{x}\sim \mathcal{D}_j}R(\Bar{x};\theta)$ the positive correlation matrix and correlation matrix, respectively. 
We start with manipulating the global objective 
\begin{gather}\label{eq:decomp}
\begin{aligned}
\mathcal{L}^{SC}(\theta; \mathcal{D}) &= -\sum_{j=1}^Jq_j Tr\{R^+_j(\theta)\} + \frac{1}{2} \norm{\sum_{j=1}^Jq_jR_j(\theta)}_F^2 \\
& = \sum_{j=1}^Jq_j \biggl( \underbrace{-Tr\{R^+_j(\theta)\}}_{\text{intra-client contraction}} + \underbrace{\frac{1}{2}q_j\norm{R_j(\theta)}_F^2}_{\text{intra-client contrast}} \\
&+ \underbrace{\frac{1}{2}(1-q_j)Tr\{R_j(\theta)R_{-j}(\theta)\}}_{\text{inter-client contrast}} \biggr),
\end{aligned}   
\end{gather}
where $R_{-j}(\theta) \triangleq \frac{1}{1-q_j}\sum_{j'\neq j}q_{j'}R_{j'}(\theta)$. From Eq. (\ref{eq:decomp}), we notice that $\mathcal{L}^{SC}(\theta; \mathcal{D})$ can be decomposed into a weighted sum of $J$ terms corresponding to $J$ clients, where each term consists of three sub-terms accounting for intra-client contraction (of positive pairs), intra-client contrast (of negative pairs), and inter-clients contrast (of negative pairs), respectively. Inspired by this decomposition, we construct the following local objective:
\begin{gather}\label{eq:local}
\begin{aligned}
\mathcal{L}^{SC}_j(\theta; \bar{R}_{-j})&=-Tr\{R^+_j(\theta)\} + \frac{1}{2}q_j\norm{R_j(\theta)}_F^2 \\
&+ (1-q_j)Tr\{R_j(\theta)\bar{R}_{-j}\},
\end{aligned}   
\end{gather}
where $\bar{R}_{-j}\in \mathbb{R}^{H\times H}$ is an estimate of $R_{-j}(\theta)$, whose updates relying on the communication with the server. Since $\bar{R}_{-j}$ is treated as a constant (stop gradient) in local objectives, we intentionally remove the coefficient $1/2$ before the third term for gradient alignment between local and global objectives. That is to say, when $\bar{R}_{-j} = R_{-j}(\theta)$, we have
\begin{gather}\label{eq:align}
\nabla \mathcal{L}^{SC}(\theta; \mathcal{D}) = \sum_{j=1}^J q_j \nabla \mathcal{L}^{SC}_j(\theta, \bar{R}_{-j}). 
\end{gather} 

Note that directly applying FedAvg results in a misalignment of gradients, which is inherited from the fact that the global objective of FedSSL does not equal to the weighted sum of local objectives as suggested in Eq. (\ref{eq:neq}).

The process of FedSC is similar to FedAvg, except sharing and aggregating local correlation matrices $\tilde{R}_j^t$ besides model weights. To begin with, the server synchronizes local models with the global model. Subsequently, clients compute their local correlation matrices and send them to the server. Following this, the server distributes the aggregated global correlation matrices back to the clients. The clients then proceed to update their local models in accordance with the local objective specified in Eq. (\ref{eq:local}). Finally, the server aggregates the local models and initiates the next iteration. The process is summarized in Fig. \ref{fig:FLarc}.

The detailed algorithm of FedSC is shown in Algorithm \ref{alg:server}. Here, clients use Algorithm \ref{alg:R} \texttt{DP-CalR} to calculate local correlation matrices to be shared $\Tilde{R}_j^t$ with differential privacy (DP) protection, which is detailed in Sec. \ref{sec:share}. During local training, clients minimize $\mathcal{L}^{SC}_j$ through stochastic gradient descent (SGD), which is detailed in Sec. \ref{sec:local}. It can be noticed that both clients and the server maintain the knowledge of global correlation matrix $\tilde{R}^t$. 

Intuitively, since the averaged local gradients align with the global gradient as shown in Eq. (\ref{eq:align}), the drift and variance of local gradients contribute $\mathcal{O}(1/T)$ and $\mathcal{O}(1/\sqrt{T})$ to the convergence rate, respectively, which has been extensively studied by previous works on FedAvg. The difference is that the shared correlation matrix $\tilde{R}_j^t$ introduces additional perturbation due to its aging (compared with instant correlation matrix $R_{-j}(\theta)$) and DP noise. The perturbation caused by aging is proportional to the movements of weights, which is proportional to the squared learning rate $\eta^2$, thus contributing an additional $\mathcal{O}(1/T)$ factor to the convergence rate. This is what motivates the design of FedSC.

\begin{algorithm}[t]
\caption{\texttt{FedSC}}
\begin{algorithmic}[1]
\STATE{\textbf{Initialization}: $\theta^0$ and a set of clients $[J]$}
\FOR{$t=1...T$}
\STATE{Server samples a subset of clients $\mathcal{J}^t\subset [J]$ }
\IF{$t=1$}
\FOR{$j\in [J]$}
\STATE{Server sends $\theta^{t-1}$ to client $j$.}
\STATE{Client $j$ uploads $\tilde{R}^t_j= \texttt{DP-CalR}(\theta^{t-1},\mathcal{D}_j)$.}
\ENDFOR
\STATE{Server sends $\tilde{R}^t = \sum_{j=1}^Jq_j\tilde{R}^t_j$ to all clients.}
\ELSE
\FOR{$j\in \mathcal{J}^t$}
\STATE{Server sends $\theta^{t-1}$ to client $j$.}
\STATE{Client $j$ uploads $\tilde{R}^t_j= \texttt{DP-CalR}(\theta^{t-1},\mathcal{D}_j)$.}
\STATE{Server updates $\tilde{R}^t = \tilde{R}^{t-1}-q_j\tilde{R}_j^{t-1}+q_j\tilde{R}_j^{t}$.}
\ENDFOR
\STATE{Server sends $\tilde{R}^t$ to all clients.}
\STATE{Server sets $\tilde{R}_j^t =\tilde{R}_j^{t-1}$ for $j\notin \mathcal{J}^t$.}
\ENDIF
\FOR{$j\in \mathcal{J}^t$}
\STATE{Client $j$ calculates $\tilde{R}^t_{-j} = \frac{1}{1-q_j}(\tilde{R}^t-q_j\tilde{R}^t_j)$.}
\STATE{Client $j$ trains local model with procedures in Sec. \ref{sec:local}, and returns updated weights $\theta_j^{t-1}$.}
\ENDFOR
\STATE{Server aggregation: $\theta^t = \frac{1}{|\mathcal{J}^t|}\sum_{j\in \mathcal{J}^t} \theta_j^{t-1}$.}
\ENDFOR
\STATE{\textbf{return}: $\theta^T$}
\end{algorithmic}
\label{alg:server}
\end{algorithm}

\begin{algorithm}[htbp]
\caption{\texttt{DP-CalR}}
\begin{algorithmic}[1]
\STATE{\textbf{Inputs}: $\theta$ and local dataset $\mathcal{D}_j$}
\STATE{$\tilde{R}^t_j = 0$}
\FOR{$\bar{x}\in \mathcal{D}_j$}
\STATE{Sample $x_1, x_2,...,x_V\sim \mathcal{A}(\cdot|\bar{x})$}
\STATE{Calculate $\check{z}_v=\texttt{NormClip}(z(x_v;\theta),\sqrt{\mu})$, for $v=1,2,...,V$. }
\STATE{$\tilde{R}^t_j = \tilde{R}^t_j + \frac{1}{|\mathcal{D}_j|V}\sum_{v=1}^V\check{z}_v\check{z}_v^T$}
\ENDFOR
\STATE{$\tilde{R}^t_j = \tilde{R}^t_j + \mathcal{N}(0,\sigma^2\mathbf{I})$.}
\STATE{\textbf{return}: $\tilde{R}^t_j$}
\end{algorithmic}
\label{alg:R}
\end{algorithm}

\subsection{Correlation Matrices Sharing}\label{sec:share}
DP protection is applied when correlation matrices are shared to mitigate additional privacy leakage on local dataset. A typical Gaussian mechanism is adopted, with parameters $\mu$ and $\sigma^2$ controlling sensitivity and noise scale, respectively. The process is summarized in Algorithm \ref{alg:R}.

\subsection{Local Training}\label{sec:local}
The local training process follows mini-batch stochastic gradient descent (SGD). At each iteration, consider a batch of $B$ samples $\bar{X}\sim \mathcal{D}_j$. Let $X_1,X_2,...,X_{2V}\sim \mathcal{A}(\bar{X})$ be $2V$ views augmented from $\bar{X}$. The empirical correlation matrices are calculated as follows: 
\begin{gather}\label{eq:emp}
\begin{aligned}
\hat{R}_j^+(\{X_v\}_v;\theta_j)&\! = \! \frac{1}{2BV}\!\sum_{v=1}^V\biggl[ \! Z(X_v;\theta_j)Z(X_{v\!+\!V};\!\theta_j)^T \\
&+Z(X_{v+V};\theta_j)Z(X_{v};\theta_j)^T\biggr];\\
\hat{R}_j(\{X_v\}_v;\theta_j)&\! = \! \frac{1}{2BV}\sum_{v=1}^{2V}Z(X_v;\theta_j)Z(X_{v};\theta_j)^T. \nonumber
\end{aligned}
\end{gather}
The batch loss $\hat{\mathcal{L}}^{SC}_j(\theta_j; \{X_v\}_v, \bar{R}_{-j})$ can be obtained by substitute $R^+_j(\theta)$ and $R_j(\theta)$ in Eq. (\ref{eq:local}) with $\hat{R}_j^+(\{X_v\}_v,\theta_j)$ and $\hat{R}_j(\{X_v\}_v,\theta_j)$, respectively. The local training follows by back-propagating the batch loss and updating the model weights iteratively.

\subsection{Comparison with existing FedSSL frameworks}
In this subsection, we discuss the privacy leakage and communication overhead of FedSC in comparison with other FedSSL frameworks.


\textbf{Sharing correlation matrices only results in negligible communication overhead:} Although FedSC shares correlation matrices in addition, it still results in less communication overhead than SOTA non-contrastive FedSSL frameworks \cite{zhuang2021collaborative,zhuang2022divergence,he2020momentum}, due to the implementation of predictor in non-contrastive SSL methods. For example, the feature dimension is $H=512$ in our experiments, thus the correlation matrices yield $H\times H \approx 260,000$ additional parameters to be communicated. In contrast, the structure of the predictor is often a three-layer multilayer perceptron (MLP), which contains parameters that are multiples of the correlation matrices. In our case, we choose a typical size of $(512-1024-512)$ resulting in $1,000,000$ parameters. The overhead 
of correlation matrices is negligible compared with the encoders. Therefore, even compared with contrastive SSL, which does not have a predictor, the communication overhead resulting from sharing correlation matrices is not a concern.

\textbf{The extra privacy leakage is probably comparable to that caused by sharing predictors:}
The predictors in non-contrastive SSL also lead to potential privacy leakages. Although theoretical characterization has not been established, recent works shed lights on the operational meaning of the predictors \cite{halvagal2023implicit,tian2021understanding}, suggesting what information is probably leaked. Particularly, \cite{tian2021understanding} reports that linear predictors in BYOL align with the correlation matrices $R(\theta)$ during training. This interesting finding suggests that predictors probably contain similar information as the correlation matrices. 


\section{Theoretical Analysis}\label{sec:the}
In this section, we first analyze the additional privacy leakage and convergence of FedSC. Our findings are summarized as follows:

    $\bullet$ We prove that sharing correlation matrices through \texttt{DP-CalR}  results in $\left( \frac{T\mu^2}{2\sigma^2}+\sqrt{\frac{2T\mu^2\log 1/\delta}{\sigma^2}},\delta\right)$-DP.
    
    $\bullet$ We provide the convergence analysis of FedSC. Specifically, with large batch size $B$, large number of views $V$ and small scale of DP noise $\sigma$, we can achieve a convergence rate close to $\mathcal{O}(1/\sqrt{T})$. 
    
    $\bullet$ The analysis indicates superior performance of FedSC over FedAvg whose convergence is dominated by a $\mathcal{O}(1)$ constant meaning error floor.

\subsection{Additional Privacy Leakage}\label{sec:privana}
In this subsection, we analyze the Gaussian mechanism in Algorithm \ref{alg:R}  \texttt{DP-CalR}. We start with definitions of variations of differential privacy (DP).
\begin{definition}[$(\epsilon,\delta)$-DP]
A mechanism $\mathcal{M}:\mathcal{X} \rightarrow \mathcal{Y}$ is $(\epsilon,\delta)$-DP, if for any neighboring $X,X'\in \mathcal{X}$ and $\mathcal{S}\subset \mathcal{Y}$, the following inequality is satisfied.
\begin{gather}
Pr(\mathcal{M}(X)\in \mathcal{S}) \leq e^\epsilon Pr(\mathcal{M}(X')\in \mathcal{S})+\delta. \nonumber
\end{gather}
\end{definition}
DP protects the inputs of a mechanism from membership inference attacks. For a mechanism satisfying DP, we expect that one can hardly tell whether the input contains a certain entry by only looking at the output. In our case, we do not want the server to know whether a local dataset contains a particular data point. 
\begin{definition}[$(\alpha,\epsilon)$-RDP \cite{mironov2017renyi}]
A mechanism $\mathcal{M}:\mathcal{X} \rightarrow \mathcal{Y}$ has $(\alpha,\epsilon)$-R{\'e}nyi differential privacy, if for any neighboring $X,X'\in \mathcal{X}$, $Y=\mathcal{M}(X)$ and $Y'=\mathcal{M}(X')$, the following inequality is satisfied:
\begin{gather}
D_\alpha(P_{Y}||P_{Y'})\leq \epsilon \nonumber,
\end{gather}
where $D_\alpha(P_{Y}||P_{Y'})$ is R{\'e}nyi-divergence of order $\alpha >1$
\begin{gather}
D_\alpha(P_{Y}||P_{Y'}) \triangleq \frac{1}{\alpha - 1}\log\mathsf{E}_{Y'} \left(\frac{P_{Y}(Y')}{P_{Y'}(Y')}\right)^\alpha. \nonumber
\end{gather}
\end{definition}
RDP is a variation of DP with many good properties, which are summarized in the following Lemmas.
\begin{lemma}[Gaussian Mechanism of RDP \cite{mironov2017renyi}]\label{lem:gaurdp}
Let $f:\mathcal{X} \rightarrow \mathbb{R}^n$ be a function with $l_2$ sensitivity $W$, then the Gaussian mechanism $G_f(\cdot) = f(\cdot) + \mathcal{N}(0,\mathbf{I}_n\sigma^2)$ is $(\alpha, \frac{\alpha W^2}{2\sigma^2})$-RDP.
\end{lemma}
\begin{lemma}[Composition of RDP \cite{mironov2017renyi} ]\label{lem:comp}
Let $\mathcal{M}_1: \mathcal{X} \rightarrow \mathcal{Y} $ be $(\alpha,\epsilon_1)$-RDP, and $\mathcal{M}_2: \mathcal{X}\times \mathcal{Y} \rightarrow \mathcal{Z}$ be $(\alpha,\epsilon_2)$-RDP. 
Then the mechanism $\mathcal{M}_3: \mathcal{X}\rightarrow \mathcal{Y}\times \mathcal{Z}, X\mapsto \left(\mathcal{M}_1(X), \mathcal{M}_2(X, \mathcal{M}_1(X))\right)$ is $(\alpha,\epsilon_1+\epsilon_2)$-RDP to $X$.
\end{lemma}
\begin{lemma}[\cite{mironov2017renyi}]\label{lem:dp}
If a mechanism is $(\alpha,\epsilon)$-RDP, then it is $(\epsilon+\frac{\log 1/\delta}{\alpha -1},\delta)$-DP.
\end{lemma}
With all these preparations, we use the following proposition to characterize the additional privacy leakage of FedSC.
\begin{proposition}[Additional Privacy Leakage of FedSC]
Sharing correlation matrices for $T_j$ times through Algorithm \ref{alg:R} \texttt{DP-CalR} results in $\left( \frac{T_j\mu^2}{2\sigma^2}+\sqrt{\frac{2T_j\mu^2\log 1/\delta}{\sigma^2}},\delta\right)$-DP.
\end{proposition}
\begin{proof}
We start with the sensitivity of \texttt{DP-CalR}.
\begin{gather}
\begin{aligned}
\norm{\tilde{R}^t_j}_F &= \norm{\frac{1}{|\mathcal{D}_j|}\sum_{\Bar{x}\in \mathcal{D}_j}\frac{1}{V}\sum_{v=1}^V\check{z}_v(\Bar{x})\check{z}_v(\Bar{x})^T}_F \\
& \leq \norm{\frac{1}{|\mathcal{D}_j|}\sum_{\Bar{x}'\in \mathcal{D}_j/\Bar{x}}\frac{1}{V}\sum_{v=1}^V\check{z}_v(\Bar{x}')\check{z}_v(\Bar{x}')^T}_F \\
&+ \frac{1}{|\mathcal{D}_j|}\frac{1}{V}\sum_{v=1}^V\norm{\check{z}_v(\Bar{x})\check{z}_v(\Bar{x})^T}_F \nonumber
\end{aligned} 
\end{gather}
where $\check{z}_v(\Bar{x})$ is the representation of the $v$-th view of data $\Bar{x}$.
Notice that for any $\Bar{x}$
\begin{gather}
\begin{aligned}
\norm{\check{z}_v(\Bar{x})\check{z}_v(\Bar{x})^T}_F &= \sqrt{Tr\left\{\check{z}_v(\Bar{x})\check{z}_v(\Bar{x})^T\check{z}_v(\Bar{x})\check{z}_v(\Bar{x})^T\right\}} \\    
& \leq \mu \nonumber
\end{aligned}
\end{gather}
The sensitivity is finally bounded by $\mu/|\mathcal{D}_j|$. With Lemma \ref{lem:gaurdp}, we have \texttt{DP-CalR} is $\left(\alpha, \frac{\alpha\mu^2}{2\sigma^2|\mathcal{D}_j|^2}\right)$-RDP. With Lemma \ref{lem:comp}, sharing correlation matrices for $T_j$ times results in  $\left(\alpha, \frac{T_j\alpha\mu^2}{2\sigma^2|\mathcal{D}_j|^2}\right)$-RDP, which is $\left( \frac{T_j\mu^2}{2\sigma^2|\mathcal{D}_j|^2}+\sqrt{\frac{2T_j\mu^2\log 1/\delta}{\sigma^2|\mathcal{D}_j|^2}},\delta\right)$-DP using Lemma \ref{lem:dp} with $\alpha = \sqrt{\frac{2\sigma^2|\mathcal{D}_j|^2\log 1/\delta}{T_j\mu^2}}+1$.
\end{proof}
From the results, we can notice that for arbitrarily $\delta$, $T_j$ and $\sigma$, the parameter $\epsilon =\frac{T_j\mu^2}{2\sigma^2|\mathcal{D}_j|^2}+\sqrt{\frac{2T_j\mu^2\log 1/\delta}{\sigma^2|\mathcal{D}_j|^2}}$ approaches to zero when the size of local dataset approaches to infinity, indicating no differential privacy leakage.

\begin{table*}
\caption{Performance comparison between FedSC and SOTAs on benchmark tasks: FedSC outperforms most of the SOTAs under different settings. Here we use bold and underline to mark the highest and second highest accuracy, respectively. }
\begin{center}
\begin{footnotesize}
\begin{tabular}{c|c|c|c||c|c|c}

\hline
        &  SVHN  & CIFAR10  &  CIFAR100 & SVHN  & CIFAR10  &  CIFAR100 \\
\hline
Participation & $5/5$ & $10/10$ & $20/20$ & $2/5$ & $2/10$ & $4/20$ \\
\hline
\hline
FedAvg + BYOL & $87.85\pm0.49$ & $68.14\pm 0.51$ & $43.54\pm 1.12$ & $87.10\pm 0.79 $ & $65.28\pm0.83 $ & $38.77\pm 1.04 $ \\
\hline
FedAvg + SC & $90.52\pm 0.42$ & $77.82\pm 0.82$ & $56.24\pm 0.19$ & $89.89\pm 0.94$ & $75.36\pm 0.36 $ & $42.95\pm 0.52$ \\
\hline
FedU & $87.92\pm 0.31$ & $68.39\pm 0.69$ & $43.81\pm 0.98$ & $87.40\pm 0.75$ & $65.52\pm 0.51 $ & $39.11\pm 0.92 $ \\
\hline
FedEMA & $\mathbf{91.87}\pm 0.30$ & $68.78\pm 0.25$ & $44.18\pm0.73 $ & $88.97\pm 0.82$ & $65.93\pm 0.63$ & $39.78\pm 1.20$ \\
\hline
FedX & $74.60\pm 0.72$ & $59.17\pm 0.93 $ & $39.70\pm 0.39$ & $73.34\pm0.88$ & $57.42\pm 0.91  $ & $33.54\pm 0.67 $ \\
\hline
FedCA & $89.92\pm 0.14$ & $78.22\pm 0.22$ & $52.35\pm 0.09 $ & $89.28\pm 0.44$ & $\mathbf{77.22}\pm0.65 $ & $51.58\pm 0.18$ \\
\hline
FedSC (Proposed) & $\underline{91.78}\pm 0.49$ & $\mathbf{80.06}\pm 0.35$ &$\mathbf{58.35}\pm 0.15$ & $\mathbf{91.03}\pm 0.58$ & $\underline{77.12}\pm 0.44$ & $\mathbf{56.64}\pm 0.65 $ \\
\hline
\hline
Centralized SC & $93.17 \pm 0.13$ & $90.21\pm 0.08$ & $64.32 \pm 0.05 $ & $-$ & $-$ & $-$ \\
\hline
\end{tabular}
\end{footnotesize}
\end{center}
\label{tab:comp1}
\vspace{-1 em}
\end{table*}

\subsection{Convergence of FedSC}
This subsection presents the convergence of FedSC. We begin with the following assumptions.
\begin{assumption}\label{ass:f1}
For any $\theta$ and $x$, NN's output is bounded in norm $||z(x,\theta)||_2\leq A_0$ for some constant $A_0$.
\end{assumption}
\begin{assumption}\label{ass:f2}
For any $\theta$ and $x$,  the Jacobian matrix of NN's output is bounded in norm $||\nabla_\theta z(x,\theta)||_F \leq A_1$ for some constant $A_1$.
\end{assumption}
\begin{assumption}\label{ass:f3}
The function represented by NN has bounded second order derivatives, i.e, for any $\theta$ and $x$,
\begin{gather}
    \sum_{m,p}\norm{\partial_m\partial_pz(x;\theta)}_2^2\leq A_2^2 \nonumber
\end{gather}
for some constant $A_2$, where $\partial_p$ refers to the partial derivation with respect to the $p$-th entry of $\theta$.
\end{assumption}
Assumption \ref{ass:f1} can be satisfied when the NN has a normalization layer at the end or uses bounded activation functions, such as sigmoid, at the output layer. Assumption \ref{ass:f2} accounts for the Lipschitz continuity of $z(x,\theta)$, which is often the case when hidden layers of a NN uses activation functions, such as tanh, sigmoid and relu. Note that Assumption \ref{ass:f2} is weaker than the bounded gradient norm assumption used in previous works \cite{li2019convergence,noble2022differentially}. However, in our case, it can lead to bounded gradient norm due to the structure of SC objectives, which is detailed in the appendices. Assumption \ref{ass:f3} accounts for the strong smoothness of the NN, which is widely adopted in existing works \cite{karimireddy2020scaffold,li2019convergence,karimireddy2020mime} . With these assumptions, we demonstrate the convergence of FedSC with the following theorem.

\begin{theorem}\label{thm:0}
Let Assumption \ref{ass:f1}, \ref{ass:f2} and \ref{ass:f3} hold. Choose $\mu > A_0^2$, and the local learning rate $\eta = \mathcal{O}\left(\frac{1}{\sqrt{TE}}\right)$, where $T$ and $E$ are the number of communication rounds and local updates, respectively. Then FedSC achieves
\begin{gather}
\begin{aligned}
&\frac{1}{TE}\sum_{t=0}^{T-1}\sum_{e=0}^{E-1}\expc{\norm{\nabla \mathcal{L}^{SC}(\theta^{t,e};\mathcal{D})}^2} \\
&\leq \mathcal{O}\Biggl(\frac{E^2(H^2\sigma^2+C_2)}{TE} + \frac{\sqrt{E}\sqrt{\frac{J/|\mathcal{J}|-1}{J-1}}}{\sqrt{T}} \\
&+ \frac{\left(\frac{1}{V}+\max_j\frac{|\mathcal{D}_j|/B-1} {|\mathcal{D}_j|-1}\right)\left(H^2\sigma^2 + C_4\right) }{\sqrt{TE}} \\
&+ \frac{1}{V}+\max_j\frac{|\mathcal{D}_j|/B-1} {|\mathcal{D}_j|-1}+H^2\sigma^2\Biggr)   
\end{aligned}    
\end{gather}
where $\theta^{t,e}\triangleq \frac{1}{|\mathcal{J}^t|}\sum_{j\in \mathcal{J}^t}\theta_j^{t,e}$ is the virtual averaged weights, and $\theta_j^{t,e}$ the local weights of client $j$ at the $e$-th step in the $t$-th round; $B$ and $V$ are batch size and number of augmented views, respectively; $C_2$ and $C_4$ here are constants only depending on $A_0$, $A_1$ and $A_2$.
\end{theorem}


\subsubsection{Superior performance of FedSC}
The convergence rate of FedSC is dominated by the following term when $T$ approaches infinity
\vspace{-0.3em}
\begin{gather}
\gamma = \mathcal{O}\left(\frac{1}{V}+\max_j\frac{|\mathcal{D}_j|/B-1} {|\mathcal{D}_j|-1}+H^2\sigma^2\right).
\end{gather}
The bias of local batch gradients results in a rate of $\bigl(\frac{1}{V}+\max_j\frac{|\mathcal{D}_j|/B-1} {|\mathcal{D}_j|-1}\bigr)$, in which specifically, sampling the data set $\mathcal{D}_j$ (without replacement) leads to the rate of $\frac{|\mathcal{D}_j|/B-1} {|\mathcal{D}_j|-1}$, and sampling the augmentation kernel $\mathcal{A}(\cdot|\bar{x})$ results in $\frac{1}{V}$. Note that this bias also exists in centralized SSL training, and does not result from federation.  Sampling variance, from the augmentation kernel, in the shared correlation matrix contributes $\frac{1}{V}$ to the convergence. The DP noise contributes $H^2\sigma^2$. If we set batch size $B=|\mathcal{D}_j|$, generate infinite number of views $V=\infty$ and not apply DP protection, i.e., $\sigma^2 = 0$, then the error floor will disappear, which results in convergence rate similar to FedAvg in supervised FL. In comparison, if we directly use the SC objective without modification and apply FedAvg, there will be a constant error floor independent with batch size $B$ and number of views $V$, due to the misalignment between the averaged local objectives and the global objectives.

\subsubsection{Sketch of Proof}
We begin with the case of full clients participation. The convergence is determined by two terms:  1) The squared norm of the bias of the averaged local gradient and 2) The variance of the averaged local gradient. The norm of bias can be factorized into three components: 
1.a) The ``drift" of the local weights, leading to a factor of $\mathcal{O}(1/T)$.
1.b) The bias in the batch gradient of local objectives $\mathcal{L}^{SC}_j(\theta; \bar{R}_{-j})$, contributes a factor of $\mathcal{O}\left(\frac{1}{V}+\max_j\frac{|\mathcal{D}j|/B-1} {|\mathcal{D}j|-1}\right)$. The bias is due to the fact that the SSL objective can not be written as a sum of samples losses like in  the supervised learning cases.
1.c) The impact of aging, sample variance, and DP noise in the shared correlation matrices $\tilde{R}_{-j}^t$. Note that $\tilde{R}_{-j}^t$ remains constant during local training. The aging (compared with $R_{-j}(\theta^{t,e})$) leads to a bias proportional to the drift of local weights, resulting in a factor of $\mathcal{O}(1/T)$. The sampling variance in $\tilde{R}_{-j}^t$ contributes a factor of  $\mathcal{O}(1/V)$. DP noise contributes a factor of $\mathcal{O}(H^2\sigma^2)$, where $H$ is the dimension of the representation. The variance in the averaged local gradient contributes a factor of $\mathcal{O}(1/\sqrt{T})$, considering 2.a) the variance in batch gradient sampling and 2.b) the DP noise in $\tilde{R}_{-j}^t$.For partial client participation, we need to consider the variance in aggregation and additional aging of $\tilde{R}_{-j}^t$. Given bounded gradient norm, the variance due to client sampling is $\mathcal{O}\left(\sqrt{E}\sqrt{\frac{J/|\mathcal{J}|-1}{T(J-1)}}\right)$. Additional aging is proportional to the extra drift, leading to a rate of $\mathcal{O}(1/T)$.

\begin{table*}
\caption{FedSC with different levels of DP protections ($\delta = 10^{-2}$): deploying DP leads to negligible performance degradation.}
\begin{center}
\begin{footnotesize}
\begin{tabular}{c|c|c||c|c}

\hline
        &  SVHN  & CIFAR10  & SVHN  & CIFAR10  \\
\hline
Participation & $5/5$ & $10/10$ & $2/5$ & $2/10$ \\
\hline
\hline
FedAvg+SC & $90.52\pm 0.42$ & $77.82\pm 0.82$ & $89.89\pm0.94$ & $75.36\pm 0.36 $  \\
\hline
$\epsilon = 3$ & $90.90\pm 0.52$ & $79.21\pm 0.59$ & $90.00\pm 0.72$ & $76.97\pm 0.64 $  \\
\hline
$\epsilon = 6$ & $91.32\pm 0.87$ & $79.42\pm 0.63$ & $90.12\pm 0.61$ & $77.08 \pm 0.46$  \\
\hline
$\epsilon = \infty$ & $91.78\pm 0.49$ & $80.06\pm 0.35$ & $91.03\pm 0.58$ & $77.12\pm 0.44 $  \\
\hline
\end{tabular}
\end{footnotesize}
\end{center}
\label{tab:comp2}
\end{table*}

\vspace{-0.6em}

\section{Experiments}
\subsection{Experimental Setup}
\textbf{Datasets:} Three datasets, SVHN, CIFAR10 and CIFAR100, are used for evaluation. SVHN is split into $5$ disjoint local datasets, each of which contains $2$ classes. CIFAR10 is split into $10$ disjoint local datasets according to the $10$ classes. CIFAR100 is split into $20$ disjoint local datasets, each of which contains $5$ classes. Therefore, the size of local datasets for SVHN, CIFAR10 and CIFAR100 tasks are $10,000$, $5,000$ and $2,500$, respectively.

\textbf{Models:} For SVHN and CIFAR10, we use a modified version of ResNet20 as backbones. For CIFAR100, the backbone is a modified version of ResNet50. 


\textbf{Hyper-parameters:} For all three tasks, the number of communication round $T=200$, and the number of local epochs is $E=5$. For SVHN and CIFAR10, the batch size is $B=512$. For CIFAR100, the batch size is $B=256$. The number of views $V=2$ for all experiments. For correlation matrices sharing, the number of views is set as $V=5$. 

\textbf{Benchmarks:} Besides FedAvg+BYOL and FedAvg+SC, we also compare with the following state of the arts: FedU \cite{zhuang2021collaborative}, FedEMA \cite{zhuang2022divergence}, FedX \cite{han2022fedx} and FedCA \cite{zhang2023federated}.

\begin{figure*}[h]
\centering    
  \subfigure[SVHN] { \label{fig:stat1}
    \includegraphics[width=0.31\linewidth]{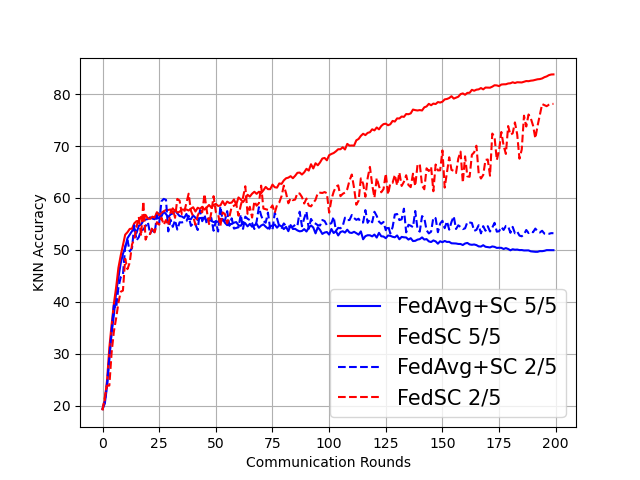}
  }
  \subfigure[CIFAR10] { \label{fig:stat6}
    \includegraphics[width=0.31\linewidth]{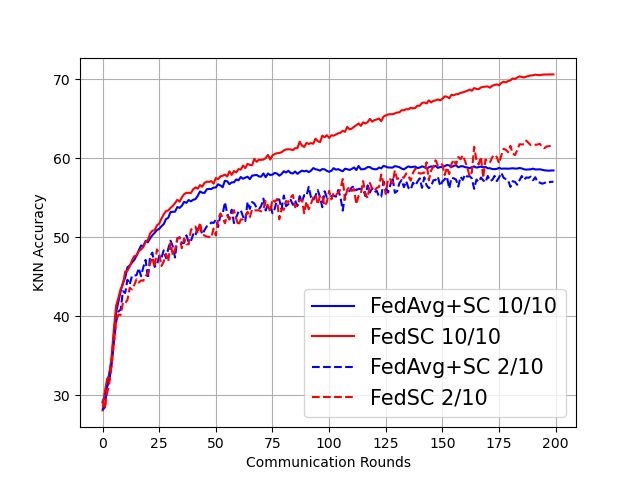}
  }
  \subfigure[CIFAR100] { \label{fig:stat6}
    \includegraphics[width=0.31\linewidth]{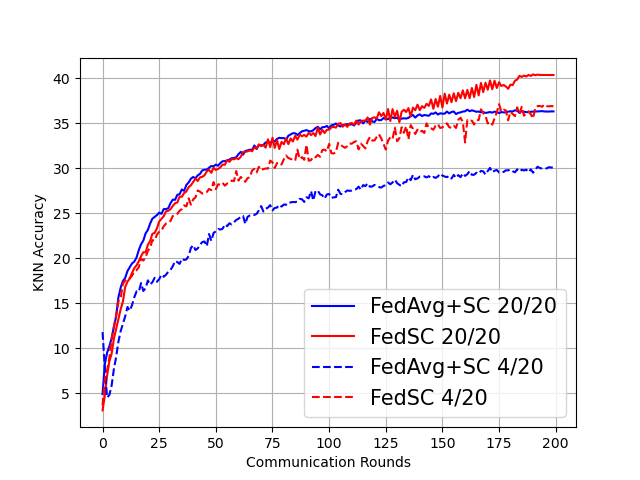}
  }
\caption{Convergence of FedSC and FedAvg+SC. 1) FedAvg+SC tends to experience either a high error floor or overfitting. 2) FedSC is able to consistently enhance KNN accuracy. This observation validates our theoretical analysis in Sec. \ref{sec:the}.}
\label{fig:conv}
\end{figure*}

\subsection{Experimental Results}

\textbf{Comparison with SOTA approaches:} Table \ref{tab:comp1} presents the performance comparisons of various algorithms under linear evaluation, where the centralized SC serves as an ideal upper bound.
We conclude the following \textbf{three observations}: (1) Our proposed algorithm, FedSC, demonstrates better or comparable performance across different tasks compared to other $6$ methods. (2) 
FedBYOL, FedU, and FedEMA show good results on SVHN but underperform on CIFAR10 and CIFAR100. We believe that this disparity is caused by the larger local dataset size in SVHN, leading to increased local updates. Since these methods incorporate momentum updates in the target encoder, a larger number of updates might be necessary to effectively initiate local training. (3) FedSC and FedCA exhibit less performance degradation when switched to the partial client participation case. 
We believe this is because clients in both FedSC and FedCA have extra global information about representations. Additionally, predictors in FedBYOL, FedU, and FedEMA are under the effect of client sampling, hindering their global information provision.



\begin{table}[htbp]
\caption{FedSC with different levels of DP protections ($\delta = 10^{-2}$): deploying DP leads to negligible performance degradation.}
\begin{center}
\begin{footnotesize}
\begin{tabular}{c|c||c}
\hline
  & \multicolumn{2}{c}{CIFAR100} \\
\hline
Participation & $20/20$ & $4/20$ \\
\hline
\hline
FedAvg+SC & $56.24\pm 0.19$ & $42.95\pm 0.52$   \\
\hline
$\epsilon = 6$ & $57.10\pm 0.82$ & $54.87\pm 0.62$   \\
\hline
$\epsilon = 12$ & $57.63\pm 0.57$ & $55.76\pm 0.58$   \\
\hline
$\epsilon = \infty$ & $58.35\pm 0.15$ & $56.64\pm 0.65$   \\
\hline
\end{tabular}
\end{footnotesize}
\end{center}
\label{tab:comp3}
\end{table}

\begin{table}
\caption{FedSC with different levels of DP protections ($\delta = 10^{-4}$): deploying DP leads to negligible performance degradation.}
\begin{center}
\begin{footnotesize}
\begin{tabular}{c|C{1.6cm}|C{1.6cm}|C{1.6cm}}

\hline
        &  SVHN  & CIFAR10  & CIFAR100  \\
\hline
Participation & $2/5$ & $2/10$ & $4/20$  \\
\hline
\hline
FedAvg+SC & $89.89\pm0.94$ & $75.36\pm 0.36 $ & $42.95\pm 0.52 $ \\
\hline
$\epsilon = 3$ & $89.95\pm 0.81$ & $76.75\pm 0.62$ & $54.22\pm 0.72$\\
\hline
$\epsilon = 8$ &  $90.12\pm0.61$ & $77.08 \pm 0.46$ & $54.87 \pm 0.62$\\ 
\hline
$\epsilon = \infty$ & $91.03\pm 0.58$ & $77.12\pm 0.44 $ & $56.64\pm 0.65$ \\
\hline
\end{tabular}
\end{footnotesize}
\end{center}
\label{tab:comp3.1}
\end{table}

\textbf{DP Impact}: Table \ref{tab:comp2}, \ref{tab:comp3} and \ref{tab:comp3.1} illustrate the impact of the DP mechanism on FedSC's performance. It is shown that with a reasonable degree of DP protection, there is only a modest decline in FedSC's performance, which remains better than that of FedAvg+SC. Given that our focus is on data level DP, the extra privacy leakage shown in the tables is typically insignificant when compared to the leakage resulting from the encoders. On the other hand, according to the analysis in Sec. \ref{sec:privana}, a smaller dataset necessitates a higher level of DP noise to maintain the same degree of privacy protection. The local dataset sizes for SVHN, CIFAR10, and CIFAR100 tasks are $10,000$, $5,000$, and $2,500$, respectively. As a result, for the CIFAR100 task, we choose a slightly higher privacy budget compared to the other two tasks.

\textbf{Convergence:} Fig. \ref{fig:conv} compares the convergence of proposed FedSC and FedAvg+SC, in terms of communication rounds and KNN accuracy. The figures reveal that FedAvg+SC tends to experience either a high error rate or overfitting as the number of communication rounds grows. In contrast, FedSC can consistently enhance KNN accuracy. This validates our theoretical analysis in Sec. \ref{sec:the}.

\section{Conclusion}
In this paper, we proposed FedSC, a novel FedSSL framework based on spectral contrastive objectives. In FedSC, clients share correlation matrices besides local weights periodically. With shared correlation matrices, clients are able to contrast inter-client sample contrast in addition to intra-client contrast and contraction. To mitigate the extra privacy leakage on local dataset, we adopted DP mechanism on shared correlation matrices. We provided theoretical analysis on privacy leakage and convergence, demonstrating the efficacy of FedSC. To the best knowledge of the authors, this is the first provable FedSSL method.

\section{Impact Statements}
This paper presents work whose goal is to advance the field of Machine Learning. There are many potential societal consequences of our work, none of which we feel must be specifically highlighted here.





\bibliography{example_paper}
\bibliographystyle{icml2024}

\newpage
\appendix
\onecolumn

\section{Derivation of SC objective}\label{sec: SC_derivative}
\begin{gather}
\begin{aligned}
\mathcal{L}^{SC}(\theta; \mathcal{D}) &\triangleq  -\mathsf{E}_{\Bar{x}\sim \mathcal{D}}\mathsf{E}_{x,x^+\sim \mathcal{A}(\cdot|\Bar{x})}\left[z(x;\theta)^Tz(x^+;\theta)\right] + \frac{1}{2}\mathsf{E}_{x,x^-\sim \mathcal{A}(\cdot|\mathcal{D})}\left[\left(z(x;\theta)^Tz(x^-;\theta)\right)^2\right] \\
& = -\mathsf{E}_{\Bar{x}\sim \mathcal{D}}\mathsf{E}_{x,x^+\sim \mathcal{A}(\cdot|\Bar{x})}\left[Tr\left\{z(x^+;\theta)z(x;\theta)^T\right\}\right] \\
&+ \frac{1}{2}\mathsf{E}_{x,x^-\sim \mathcal{A}(\cdot|\mathcal{D})}\left[Tr\left\{z(x;\theta)z(x;\theta)^Tz(x^-;\theta)z(x^-;\theta)^T\right\}\right] \\
& = -\mathsf{E}_{\Bar{x}\sim \mathcal{D}}\left[Tr\left\{R^+(\theta)\right\}\right] + \frac{1}{2}Tr\left\{\mathsf{E}_{x\sim \mathcal{A}(\cdot|\mathcal{D})}z(x;\theta)z(x;\theta)^T\mathsf{E}_{x^-\sim \mathcal{A}(\cdot|\mathcal{D})}z(x^-;\theta)z(x^-;\theta)^T\right\} \\
& = -\mathsf{E}_{\Bar{x}\sim \mathcal{D}}\left[Tr\left\{R^+(\theta)\right\}\right] + \frac{1}{2}\norm{\mathsf{E}_{x\sim \mathcal{A}(\cdot|\mathcal{D})}z(x;\theta)z(x;\theta)^T}_F^2 \\
&= -\mathsf{E}_{\Bar{x}\sim \mathcal{D}} Tr\{R^+(\Bar{x};\theta)\} + \frac{1}{2} \norm{\mathsf{E}_{\Bar{x}\sim \mathcal{D}}R(\Bar{x};\theta)}_F^2    
\end{aligned}
\end{gather}

\section{Proof of Theorem \ref{thm:0}}
\subsection{Additional Notations}
Let $\theta_j^{t,e}$ and $v_j^{t,e}$ be the local weights and local SGD direction, respectively, at the $e$-th update in the $t$-th communication round. Denote $\theta^{t,e}\triangleq \sum_jq_j\theta_j^{t,e}$ and $v_j^{t,e}\triangleq\sum_jq_jv_j^{t,e}$ the virtual averaged weights and moving direction, respectively. Since the server aggregates periodically, we have $\theta^t = \theta^{t,0}$. For simplicity, we remove the up-script "SC" in $\mathcal{L}^{SC}(\theta)$ and $\mathcal{L}^{SC}(\theta, \tilde{R}^t_{-j})$ without ambiguity.

\subsection{Assumptions}
\begin{assumption}\label{ass:1}
For any $\theta$ and $x$, NN's output is bounded in norm $||z(x,\theta)||_2<A_0$.
\end{assumption}
\begin{assumption}\label{ass:2}
For any $\theta$ and $x$,  the Jacobin of NN's output is bounded in norm $||\nabla z(x,\theta)||_F<A_1$.
\end{assumption}
\begin{assumption}\label{ass:3}
The function represented by NN has bounded second order derivatives, i.e, for any $\theta$ and $x$
\begin{gather}
    \sum_{m,p}\norm{\partial_m\partial_pz(x;\theta)}_2^2\leq A_2^2
\end{gather}
\end{assumption}

\subsection{Lemmas}
\begin{lemma}\label{lem:cov}
For any $\bar{x}$, $\{X_v\}_v$, $\theta$ and $j$, the following inequalities hold 
\begin{gather}
\norm{R(\bar{x},\theta)}_F^2, \norm{R(\theta)}_F^2, \norm{R_j(\theta)}_F^2, \norm{\hat{R}(\{X_v\}_v,\theta)}_F^2 \leq A_0^4   
\end{gather}
\begin{gather}
\sum_p\norm{\partial_p R(\bar{x},\theta)}_F^2, \sum_p\norm{\partial_p R(\theta)}_F^2, \sum_p\norm{\partial_p R_j(\theta)}_F^2, \sum_p\norm{\partial_p \hat{R}(\{X_v\}_v,\theta)}_F^2 \leq 4A_1^2A_0^2.   
\end{gather}
\end{lemma}
\begin{proof}
\begin{gather}
\begin{aligned}
\norm{R(\bar{x},\theta)}_F^2 &= \norm{\mathsf{E}_{x\sim\mathcal{A}(\cdot |\bar{x})}z(x;\theta)z^T(x;\theta)}_F^2 \\
&\leq \mathsf{E}_{x\sim\mathcal{A}(\cdot |\bar{x})} \norm{z(x;\theta)}^2  \norm{z(x;\theta)}^2  \\
&\leq A_0^4
\end{aligned}
\end{gather}  
\begin{gather}
\begin{aligned}
\sum_p\norm{\partial_p R(\bar{x},\theta)}_F^2 &\leq \sum_p\mathsf{E}_{x\sim\mathcal{A}(\cdot |\bar{x})}\norm{\partial_p z(x;\theta)z^T(x;\theta)+ z(x;\theta)\partial_pz^T(x;\theta)}_F^2 \\
&\leq 4\sum_p\mathsf{E}_{x\sim\mathcal{A}(\cdot |\bar{x})} \norm{\partial_p z(x;\theta)}^2  \norm{z(x;\theta)}^2  \\
&\leq 4A_1^2A_0^2
\end{aligned}
\end{gather}     
The remaining results directly follows Jansen's inequality.
\end{proof}

\begin{lemma}\label{lem:u}
The following function, whose $p$-th entry is defined as 
\begin{gather}
u^{t,e}_j(\theta)[p]=  -\partial_pTr\left\{R_j^+(\theta)\right\}+Tr\left\{ \partial_p R_j(\theta)\sum_{j'}q_{j'}R_{j'}(\theta)\right\}    
\end{gather}
is $\beta$-Lipschitz continuous with 
\begin{gather}
\beta^2=24(A_0^4+1)(A_1^4+A_2^2A_0^2)+48A_1^4A_0^4.
\end{gather}    
\end{lemma}
\begin{proof}
We start with the derivative of $u^{t,e}_j(\theta)[p]$
\begin{gather}
\begin{aligned}
\partial_m u^{t,e}_j(\theta)[p] &= -\partial_m\partial_pTr\left\{R_j^+(\theta)\right\}+Tr\left\{\partial_m\partial_p R_j(\theta)\sum_{j'}q_{j'}R_{j'}(\theta)\right\}+Tr\left\{\partial_p R_j(\theta)\sum_{j'}q_{j'}\partial_mR_{j'}(\theta)\right\}
\end{aligned}
\end{gather}
Using AM-GM, we have
\begin{gather}
\begin{aligned}
(\partial_m u^{t,e}_j(\theta)[p])^2 \leq 3\left(\partial_m\partial_pTr\left\{R_j^+(\theta)\right\}\right)^2+3Tr\left\{\partial_m\partial_p R_j(\theta)\sum_{j'}q_{j'}R_{j'}(\theta)\right\}^2+3Tr\left\{\partial_p R_j(\theta)\sum_{j'}q_{j'}\partial_mR_{j'}(\theta)\right\}^2
\end{aligned}
\end{gather}
For the first term, recall the definition of $R_j^+(\theta)$, we have
\begin{gather}
\begin{aligned}
\partial_m\partial_pTr\{R_j^+(\theta)\}&=2\mathsf{E}_{\bar{x}\sim\mathcal{D}_j}\partial_p\mathsf{E}_{x\sim\mathcal{A}(\cdot|\bar{x})}z^T(x;\theta)\partial_m\mathsf{E}_{x\sim\mathcal{A}(\cdot|\bar{x})}z(x;\theta) \\
&+2\mathsf{E}_{\bar{x}\sim\mathcal{D}_j}\partial_m\partial_p\mathsf{E}_{x\sim\mathcal{A}(\cdot|\bar{x})}z^T(x;\theta)\mathsf{E}_{x\sim\mathcal{A}(\cdot|\bar{x})}z(x;\theta).
\end{aligned}    
\end{gather}
Consequently, we have
\begin{gather}\label{eq:lemcov_t1}
\begin{aligned}
&\sum_{m,p}(\partial_m\partial_pTr\left\{R_j^+(\theta)\right\})^2\\
&\leq 8\left(\mathsf{E}_{\bar{x}\sim\mathcal{D}_j}\partial_p\mathsf{E}_{x\sim\mathcal{A}(\cdot|\bar{x})}z^T(x;\theta)\partial_m\mathsf{E}_{x\sim\mathcal{A}(\cdot|\bar{x})}z(x;\theta)\right)^2+8\left(\mathsf{E}_{\bar{x}\sim\mathcal{D}_j}\partial_m\partial_p\mathsf{E}_{x\sim\mathcal{A}(\cdot|\bar{x})}z^T(x;\theta)\mathsf{E}_{x\sim\mathcal{A}(\cdot|\bar{x})}z(x;\theta)\right)^2\\
&\leq 8 \sum_{m,p}\mathsf{E}_{\bar{x}\sim\mathcal{D}_j}\mathsf{E}_{x_1,x_2\sim\mathcal{A}(\bar{x})}\left(\left(\partial_pz^T(x_1;\theta)\partial_mz(x_2;\theta)\right)^2 +\left(\partial_m\partial_pz^T(x_1;\theta)z(x_2;\theta)\right)^2\right)\\
&\leq 8 \sum_{m,p}\mathsf{E}_{\bar{x}\sim\mathcal{D}_j}\mathsf{E}_{x_1,x_2\sim\mathcal{A}(\bar{x})}\left(\norm{\partial_pz(x_1;\theta)}_2^2\norm{\partial_mz(x_2;\theta)}_2^2 +\norm{\partial_m\partial_pz(x_1;\theta)}_2^2\norm{z(x_2;\theta)}_2^2\right) \\
&\leq 8 (A_1^4+A_2^2A_0^2)
\end{aligned}    
\end{gather}
where the first inequality uses AM-GM, and the second inequality uses Jensen's inequality. 

For the second term, we have
\begin{gather}\label{eq:lem1}
\begin{aligned}
Tr\left\{\partial_m\partial_p R_j(\theta)\sum_{j'}q_{j'}R_{j'}(\theta)\right\}^2\leq \norm{\partial_m\partial_p R_j(\theta)}_F^2\norm{\sum_{j'}q_{j'}R_{j'}(\theta)}_F^2.
\end{aligned}    
\end{gather}
Notice that
\begin{gather}
\begin{aligned}
\partial_m\partial_p R_j(\theta) &= \mathsf{E}_{x\sim \mathcal{A}(\cdot|\mathcal{D}_j)}\partial_m\partial_pz(x;\theta)z^T(x;\theta) + z(x;\theta)\partial_m\partial_pz^T(x;\theta) \\
&+\partial_mz(x;\theta)\partial_pz^T(x;\theta) + \partial_pz(x;\theta)\partial_mz^T(x;\theta)
\end{aligned}    
\end{gather}
We have
\begin{gather}\label{eq:lem2}
\begin{aligned}
\sum_{m,p}\norm{\partial_m\partial_p R_j(\theta)}_F^2 &\leq  4\norm{\mathsf{E}_{x\sim \mathcal{A}(\cdot|\mathcal{D}_j)}\partial_m\partial_pz(x;\theta)z^T(x;\theta)}_F^2 + 4\norm{\mathsf{E}_{x\sim \mathcal{A}(\cdot|\mathcal{D}_j)}z(x;\theta)\partial_m\partial_pz^T(x;\theta)}_F^2 \\
&+4\norm{\mathsf{E}_{x\sim \mathcal{A}(\cdot|\mathcal{D}_j)}\partial_mz(x;\theta)\partial_pz^T(x;\theta)}_F^2 + 4\norm{\mathsf{E}_{x\sim \mathcal{A}(\cdot|\mathcal{D}_j)}\partial_pz(x;\theta)\partial_mz^T(x;\theta)}_F^2 \\
&\leq  4\mathsf{E}_{x\sim \mathcal{A}(\cdot|\mathcal{D}_j)}\biggl[\norm{\partial_m\partial_pz(x;\theta)z^T(x;\theta)}_F^2 + \norm{z(x;\theta)\partial_m\partial_pz^T(x;\theta)}_F^2 \\
&+\norm{\partial_mz(x;\theta)\partial_pz^T(x;\theta)}_F^2 + \norm{\partial_pz(x;\theta)\partial_mz^T(x;\theta)}_F^2\biggr] \\
&\leq 8 \sum_{m,p}\mathsf{E}_{x\sim \mathcal{A}(\cdot|\mathcal{D}_j)}\left(\norm{\partial_m\partial_pz(x;\theta)}^2_2\norm{z(x;\theta)}^2_2+\norm{\partial_mz(x;\theta)}_2^2\norm{\partial_pz(x;\theta)}_2^2\right)\\
&\leq 8(A_1^4+A_2^2A_0^2)
\end{aligned}    
\end{gather}
Apply Lemma \ref{lem:cov}, we have
\begin{gather}\label{eq:lem3}
\begin{aligned}
\norm{\sum_{j'}q_{j'}R_{j'}(\theta)}_F^2&\leq \sum_{j'}q_{j'}\norm{R_{j'}(\theta)}_F^2\leq A_0^4
\end{aligned}    
\end{gather}
Substitute eq. (\ref{eq:lem1}) with eq. (\ref{eq:lem2}) and eq. (\ref{eq:lem3}), we have
\begin{gather}\label{eq:lemcov_t2}
\begin{aligned}
Tr\left\{\partial_m\partial_p Rj(\theta)\sum_{j'}q_{j'}R_{j'}(\theta)\right\}^2\leq 8A_0^4(A_1^4+A_2^2A_0^2)
\end{aligned}    
\end{gather}
For the third term, we have
\begin{gather}\label{eq:lemcov_t3}
\begin{aligned}
\sum_{m,p}Tr\left\{\partial_p R_j(\theta)\sum_{j'}q_{j'}\partial_mR_{j'}(\theta)\right\}^2&\leq \sum_{m,p}\norm{\partial_p R_j(\theta)}_F^2 \norm{\sum_{j'}q_{j'}\partial_mR_{j'}(\theta)}_F^2 \\
&\leq \sum_{m,p}\norm{\partial_p R_j(\theta)}_F^2 \sum_{j'}q_{j'}\norm{\partial_mR_{j'}(\theta)}_F^2\\
&\leq 16A_1^4A_0^4
\end{aligned}
\end{gather}
where the last inequality uses Lemma \ref{lem:cov}.

Combine eq. (\ref{eq:lemcov_t1}), eq. (\ref{eq:lemcov_t2}) and eq. (\ref{eq:lemcov_t3}), we have
\begin{gather}
\norm{\nabla u^{t,e}_j(\theta)}_F^2=\sum_{m,p}(\partial_m u^{t,e}_j(\theta)[p] )^2\leq 24(A_0^4+1)(A_1^4+A_2^2A_0^2)+48A_1^4A_0^4
\end{gather}
and thus $\beta^2 = 4(A_0^4+1)(A_1^4+A_2^2A_0^2)+48A_1^4A_0^4$.
\end{proof}

\begin{corollary}\label{cor:smooth}
The global loss $\mathcal{L}(\theta)$ is $\beta$-smooth.
\end{corollary}
\begin{proof}
Notice that $\nabla \mathcal{L}(\theta) = \sum_jq_ju^{t,e}_j(\theta^{t,e})$. The result follows after applying Lemma \ref{lem:u}.
\end{proof}

\begin{lemma}\label{lem:vxy}
For any random matrix $X, Y$, we have
\begin{gather}
\mathsf{Var}\left[ Tr\left\{XY\right\}\right] \leq  2 \norm{\mathsf{E}[Y]}_F^2\mathsf{Var}[X]+ 2\norm{\mathsf{E}[X]}_F^2\mathsf{Var}[Y]  
\end{gather}
\end{lemma}
\begin{proof}
\begin{gather}
\begin{aligned}
\mathsf{Var}\left[ Tr\left\{XY\right\}\right] &\leq 2\mathsf{Var}[Tr\{(X-\mathsf{E}[X])Y\}] + 2\mathsf{Var}[Tr\{\mathsf{E}[X]Y\}] \\
& \leq 2\mathsf{E}[Tr\{(X-\mathsf{E}[X])Y\}^2] + 2\norm{\mathsf{E}[X]}_F^2\mathsf{Var}[Y] \\
& \leq 2 \norm{\mathsf{E}[Y]}_F^2\mathsf{Var}[X]+ 2\norm{\mathsf{E}[X]}_F^2\mathsf{Var}[Y]
\end{aligned}
\end{gather}    
\end{proof}

\begin{lemma}\label{lem:bdgd}
The local stochastic gradient with $p$-th entry defined as 
\begin{gather}
\begin{aligned}
v_j^{t,e}[p]&=-\partial_pTr\left\{\hat{R}_j^+(\{X_v\}_v,\theta^{t,e}_j)\right\}+q_jTr\left\{\hat{R}_j(\{X_v\}_v,\theta^{t,e}_j)\partial_p\hat{R}_j(\{X_v\}_v,\theta^{t,e}_j)\right\}\\
&+ (1-q_j)Tr\left\{\partial_p\hat{R}_j(\{X_v\}_v,\theta^{t,e}_j))\tilde{R}^t_{-j} \right\}.    
\end{aligned}
\end{gather}
has bounded norm
\begin{gather}
\begin{aligned}
\expc{\norm{v_j^{t,e}}^2|\mathcal{F}^{t,0}} \leq 12A_1^2A_0^2(H^2\sigma^2+1)+ 12A_1^2A_0^6.
\end{aligned}
\end{gather}
where $\mathcal{F}^{t,0}$ is the history before the $t$-th round; $\sigma^2$ is the variance of the DP noise and $H$ is the dimension of the representation $z(x,\theta)$.
\end{lemma}
\begin{proof}
\begin{gather}
\begin{aligned}
\sum_p(v_j^{t,e}[p])^2&\leq 2\sum_pTr\left\{\partial_p\hat{R}_j^+(\{X_v\}_v,\theta^{t,e}_j)\right\}^2\\
&+ 2\sum_pTr\left\{\partial_p\hat{R}_j(\{X_v\}_v,\theta^{t,e}_j))\left(q_j\hat{R}_j(\{X_v\}_v,\theta^{t,e}_j))+(1-q_j)\tilde{R}^t_{-j}\right) \right\}^2.    
\end{aligned}
\end{gather} 
For the first term, 
\begin{gather}
Tr\left\{\hat{R}_j^+(\{X_v\}_v,\theta^{t,e}_j)\right\} =  \frac{1}{BV}\sum_{b=1}^B\sum_{v=1}^V z^T(x_{b,v},\theta^{t,e}_j)z(x_{b,v+V},\theta^{t,e}_j).
\end{gather}
Then we have
\begin{gather}
Tr\left\{\partial_p\hat{R}_j^+(\{X_v\}_v,\theta^{t,e}_j)\right\} =  \frac{1}{BV}\sum_{b=1}^B\sum_{v=1}^V \partial_pz(x_{b,v},\theta^{t,e}_j)z(x_{b,v+V},\theta^{t,e}_j)^T + z(x_{b,v},\theta^{t,e}_j)^T\partial_p z(x_{b,v+V},\theta^{t,e}_j)
\end{gather}
and
\begin{gather}
\begin{aligned}
&\sum_pTr\left\{\partial_p\hat{R}_j^+(\{X_v\}_v,\theta^{t,e}_j)\right\}^2 \\
&\leq \sum_p\frac{2}{BV}\sum_{b=1}^B\sum_{v=1}^V \norm{\partial_pz(x_{b,v},\theta^{t,e}_j)}^2\norm{z(x_{b,v+V},\theta^{t,e}_j)}^2 + \norm{z(x_{b,v},\theta^{t,e}_j)}^2\norm{\partial_p z(x_{b,v+V},\theta^{t,e}_j)}^2  \\
& \leq 4A_1^2A_0^2
\end{aligned}
\end{gather}
where the line uses Jensen's inequality and AM-GM. For the second term,
\begin{gather}
\begin{aligned}
&\sum_pTr\left\{\partial_p\hat{R}_j(\{X_v\}_v,\theta^{t,e}_j))\left(q_j\hat{R}_j(\{X_v\}_v,\theta^{t,e}_j))+(1-q_j)\tilde{R}^t_{-j}\right) \right\}^2\\
&\leq \norm{q_j\hat{R}_j(\{X_v\}_v,\theta^{t,e}_j))+(1-q_j)\tilde{R}^t_{-j}}_F^2\sum_p\norm{\partial_p\hat{R}_j(\{X_v\}_v,\theta^{t,e}_j)}_F^2  \\
& \leq \norm{q_j\hat{R}_j(\{X_v\}_v,\theta^{t,e}_j))+(1-q_j)\tilde{R}^t_{-j}}_F^24A_1^2A_0^2
\end{aligned}    
\end{gather}
where the last inequality uses Lemma \ref{lem:cov}. Combine the above results we have
\begin{gather}
\begin{aligned}
\expc{\norm{v_j^{t,e}}^2|\mathcal{F}^{t,0}} &\leq 8A_1^2A_0^2 + 8A_1^2A_0^2\expc{\norm{q_j\hat{R}_j(\{X_v\}_v,\theta^{t,e}_j))+(1-q_j)\tilde{R}^t_{-j}}_F^2|\mathcal{F}^{t,0}} \\   
&= 8A_1^2A_0^2 + 8A_1^2A_0^2\left(A_0^4 + \sum_{j'\neq j}q_{j'}^2H^2\sigma^2\right) \\   
&\leq  8A_1^2A_0^2 + 8A_1^2A_0^2\left(A_0^4 + H^2\sigma^2\right) 
\end{aligned}
\end{gather}
where we use the fact that $q_j\hat{R}_j(\{X_v\}_v,\theta^{t,e}_j))+(1-q_j)\tilde{R}^t_{-j}$ is essentially a correlation matrix plus DP noise with scale $\sum_{j'\neq j}q_{j'}^2H^2\sigma^2$.
\end{proof}

\subsection{Proof of the fully participation}
From the $\beta$-smoothness of $\mathcal{L}(\theta)$ given by Corollary \ref{cor:smooth}, we have
\begin{gather}
\mathcal{L}(\theta^{t,e+1})\leq  \mathcal{L}(\theta^{t,e}) - \eta\langle \nabla \mathcal{L}(\theta^{t,e}), v^{t,e}\rangle + \frac{\beta \eta^2 }{2}\norm{v^{t,e}}^2.
\end{gather}
Denote the history of the optimization process as $\mathcal{F}^{t,e}$, then we have
\begin{gather}
\expc{\mathcal{L}(\theta^{t,e+1})|\mathcal{F}^{t,e}} \leq \mathcal{L}(\theta^{t,e}) - \eta \langle \nabla \mathcal{L}(\theta^{t,e}), \expc{v^{t,e}|\mathcal{F}^{t,e}}\rangle + \frac{\beta \eta^2}{2}\expc{\norm{v^{t,e}}^2|\mathcal{F}^{t,e}}.
\end{gather}
Let $\bar{v}^{t,e}=\expc{v^{t,e}|\mathcal{F}^{t,e}}$ and $\bar{v}_j^{t,e}=\expc{v_j^{t,e}|\mathcal{F}^{t,e}}$, we have
\begin{gather}
\begin{aligned}
\expc{\mathcal{L}(\theta^{t,e+1})|\mathcal{F}^{t,e}}&\leq \mathcal{L}(\theta^{t,e})+ \frac{\beta \eta^2 }{2}\expc{\norm{v^{t,e}}^2|\mathcal{F}^{t,e}} \\
&+\frac{\eta}{2}\left[ -\norm{\nabla \mathcal{L}(\theta^{t,e})}^2-\norm{\bar{v}^{t,e}}^2 +\norm{\nabla \mathcal{L}(\theta^{t,e}) - \bar{v}^{t,e}}^2 \right]. 
\end{aligned}
\end{gather}
By the choice of $\eta\leq 1/\beta$, we have
\begin{gather}\label{eq:rec}
\begin{aligned}
\expc{\mathcal{L}(\theta^{t,e+1})|\mathcal{F}^{t,e}}\leq \mathcal{L}(\theta^{t,e}) - \frac{\eta}{2} \norm{\nabla \mathcal{L}(\theta^{t,e})}^2+\frac{\eta}{2}\underbrace{\norm{\nabla \mathcal{L}(\theta^{t,e}) - \bar{v}^{t,e}}^2}_{T_1} + \frac{\beta \eta^2 }{2}\underbrace{\mathsf{Var}\left[v^{t,e}|\mathcal{F}^{t,e} \right]}_{T_2}
\end{aligned}
\end{gather}

\subsubsection{Bounding the term $T_1$}
Recall the definition of local batch loss 
\begin{gather}
\hat{\mathcal{L}}_j^{SC}(\theta^{t,e}_j) =-Tr\{\hat{R}_j^+(\{X_v\}_v,\theta^{t,e}_j)\} + \frac{1}{2}q_j\norm{\hat{R}_j(\{X_v\}_v,\theta^{t,e}_j)}_F^2 
+ (1-q_j)Tr\{\hat{R}_j(\{X_v\}_v,\theta^{t,e}_j)\tilde{R}^t_{-j}\} 
\end{gather}
The $p$-th entry of $v_j^{t,e} = \nabla\hat{\mathcal{L}}_j^{SC}(\theta^{t,e}_j)$ is 
\begin{gather}\label{eq:vp}
\begin{aligned}
v_j^{t,e}[p]&=-\partial_pTr\left\{\hat{R}_j^+(\{X_v\}_v,\theta^{t,e}_j)\right\}+q_jTr\left\{\hat{R}_j(\{X_v\}_v,\theta^{t,e}_j)\partial_p\hat{R}_j(\{X_v\}_v,\theta^{t,e}_j)\right\}\\
&+ (1-q_j)Tr\left\{\partial_p\hat{R}_j(\{X_v\}_v,\theta^{t,e}_j))\tilde{R}^t_{-j} \right\}.
\end{aligned}
\end{gather}
Take expectation over $\{X_v\}_v$, we have
\begin{gather}\label{eq:vpbar}
\begin{aligned}
\bar{v}_j^{t,e}[p]&=-\partial_pTr\left\{R_j^+(\theta^{t,e}_j)\right\}+q_j\mathsf{E}_{\{X_v\}_v}\left[Tr\left\{\hat{R}_j(\{X_v\}_v,\theta^{t,e}_j)\partial_p\hat{R}_j(\{X_v\}_v,\theta^{t,e}_j)\right\}\right]\\
&+ (1-q_j)Tr\left\{\partial_pR_j(\theta^{t,e}_j)\tilde{R}^t_{-j} \right\}.
\end{aligned}
\end{gather}

The $p$-th entry of the global loss gradient is
\begin{gather}
\begin{aligned}
\partial_p \mathcal{L}(\theta^{t,e}) &= -\sum_jq_j\partial_pTr\left\{R_j^+(\theta^{t,e})\right\} + \sum_jq_jTr\left\{ \partial_p R_j(\theta^{t,e})\sum_{j'}q_{j'}R_{j'}(\theta^{t,e})\right\} \\
& = \sum_jq_j\left(-\partial_pTr\left\{R_j^+(\theta^{t,e})\right\} + q_jTr\left\{ \partial_p R_j(\theta^{t,e})R_{j}(\theta^{t,e})\right\} + Tr\left\{ \partial_p R_j(\theta^{t,e})\sum_{j'\neq j}q_{j'}R_{j'}(\theta^{t,e})\right\}\right) \\
\end{aligned}
\end{gather}
Decompose $v_j^{t,e}[p] = u^{t,e}_j(\theta^{t,e}_j)[p]+q_jb_j^{t,e}[p]+c_j^{t,e}[p]$, where the terms are defined as follows.
\begin{gather}
\begin{aligned}
u^{t,e}_j(\theta)[p] &=  -\partial_pTr\left\{R_j^+(\theta)\right\}+Tr\left\{ \partial_p R_j(\theta)\sum_{j'}q_{j'}R_{j'}(\theta)\right\} \\
b_j^{t,e}[p] &= \mathsf{E}_{\{X_v\}_v}\left[Tr\left\{\hat{R}_j(\{X_v\}_v,\theta^{t,e}_j)\partial_p\hat{R}_j(\{X_v\}_v,\theta^{t,e}_j)\right\}\right] - Tr\left\{R_j(\theta_j^{t,e})\partial_p R_j(\theta_j^{t,e})\right\}  \\
c_j^{t,e}[p] & = (1-q_j)Tr\left\{\partial_pR_j(\theta_j^{t,e})\tilde{R}^t_{-j}  \right\}-Tr\left\{\partial_pR_j(\theta_j^{t,e})\sum_{j'\neq j}q_{j'}R_{j'}(\theta_j^{t,e})  \right\}
\end{aligned}
\end{gather}
Then we have
\begin{gather}
\begin{aligned}
\bar{v}^{t,e}[p]-\partial_p \mathcal{L}(\theta^{t,e}) &=\sum_jq_j\left(u^{t,e}_j(\theta^{t,e}_j)[p]-u^{t,e}_j(\theta^{t,e})[p]+q_jb_j^{t,e}[p] + c_j^{t,e}[p]\right).
\end{aligned}    
\end{gather}
The term $T_1$ can be written as follows
\begin{gather}\label{eq:T1}
\begin{aligned}
T_1 &= \sum_p \left(\bar{v}^{t,e}[p]-\partial_p \mathcal{L}(\theta^{t,e})\right)^2 \\
& = \sum_p \left( \sum_jq_j \left(u^{t,e}_j(\theta^{t,e}_j)[p]-u^{t,e}_j(\theta^{t,e})[p]+q_jb_j^{t,e}[p] + c_j^{t,e}[p]\right)  \right)^2 \\
&\leq \sum_jq_j \sum_p\left( u^{t,e}_j(\theta^{t,e}_j)[p]-u^{t,e}_j(\theta^{t,e})[p]+q_jb_j^{t,e}[p] + c_j^{t,e} [p]\right)^2 \\
&\leq 3\sum_jq_j \left[\underbrace{\sum_p (u^{t,e}_j(\theta^{t,e}_j)[p]-u^{t,e}_j(\theta^{t,e})[p])^2}_{T_3}+q^2_j\underbrace{\sum_p(b_j^{t,e}[p])^2}_{T_4} + \underbrace{\sum_p(c_j^{t,e}[p])^2}_{T_5}\right] \\
\end{aligned}    
\end{gather}
where the third line uses Jensen's inequality and the last line uses AM-GM.
By Lemma , we have
\begin{gather}\label{eq:T3}
T_3\leq \beta^2 \norm{\theta^{t,e}-\theta_j^{t,e}}^2.  
\end{gather}
For the term $T_4$, we have
\begin{gather}
\begin{aligned}
T_4 &= \sum_pTr\left\{ \mathsf{E}_{\{X_v\}_v}\left[\hat{R}_j(\{X_v\}_v,\theta^{t,e}_j) - R_j(\theta_j^{t,e})\right] \left[\partial_p\hat{R}_j(\{X_v\}_v,\theta^{t,e}_j) - \partial_p R_j(\theta_j^{t,e}) \right]\right\}^2\\
&\leq  \sum_p\mathsf{E}_{\{X_v\}_v} \left[\norm{\hat{R}_j(\{X_v\}_v,\theta^{t,e}_j) - R_j(\theta_j^{t,e})}_F^2 \norm{\partial_p\hat{R}_j(\{X_v\}_v,\theta^{t,e}_j) - \partial_p R_j(\theta_j^{t,e})}_F^2 \right] \\
&\leq  2\mathsf{E}_{\{X_v\}_v} \left[\left(\norm{\hat{R}_j(\{X_v\}_v,\theta^{t,e}_j)}_F^2 + \norm{R_j(\theta_j^{t,e})}_F^2 \right)\sum_p\norm{\partial_p\hat{R}_j(\{X_v\}_v,\theta^{t,e}_j) - \partial_p R_j(\theta_j^{t,e})}_F^2 \right] \\
&\leq 4A_0^4\mathsf{E}_{\{X_v\}_v}\sum_p\norm{\partial_p\hat{R}_j(\{X_v\}_v,\theta^{t,e}_j) - \partial_p R_j(\theta_j^{t,e})}_F^2\\
&=4A_0^4\mathsf{E}_{\bar{X}\sim \mathcal{D}_j}\mathsf{E}_{\{X_v\}_v|\bar{X}}\sum_p\norm{\partial_p\hat{R}_j(\{X_v\}_v,\theta^{t,e}_j)-\partial_p R(\bar{X}_j^{t,e};\theta_j^{t,e}) + \partial_p R(\bar{X}_j^{t,e};\theta_j^{t,e})- \partial_p R_j(\theta_j^{t,e})}_F^2    
\end{aligned}    
\end{gather}
where the third inequality uses Lemma \ref{lem:cov} $\bar{X}$ is a batch of samples drawn from $\mathcal{D}_j$, and $\{X_v\}_v$ are augmented views of $\bar{X}$. Notice that 
\begin{gather}
\mathsf{E}_{\{X_v\}_v|\bar{X}}\langle\partial_p\hat{R}_j(\{X_v\}_v,\theta^{t,e}_j)-\partial_p R(\bar{X}_j^{t,e};\theta_j^{t,e}), \partial_p R(\bar{X}_j^{t,e};\theta_j^{t,e})- \partial_p R_j(\theta_j^{t,e})\rangle = 0
\end{gather}
we have
\begin{gather}\label{eq:T4}
\begin{aligned}
T_4 & \leq   4A_0^4\sum_p\mathsf{E}_{\bar{X}\sim \mathcal{D}_j}\mathsf{E}_{\{X_v\}_v|\bar{X}}\biggl(\norm{\partial_p\hat{R}_j(\{X_v\}_v,\theta^{t,e}_j)-\partial_p R(\bar{X}_j^{t,e};\theta_j^{t,e})}_F^2 + \norm{\partial_p R(\bar{X}_j^{t,e};\theta_j^{t,e})- \partial_p R_j(\theta_j^{t,e})}_F^2\biggr) \\
& = 4A_0^4  \sum_p\left(\mathsf{E}_{\bar{X}\sim \mathcal{D}_j}\mathsf{Var}_{\{X_v\}_v|\bar{X}} \left[\partial_p\hat{R}_j(\{X_v\}_v,\theta^{t,e}_j)\right]+\mathsf{Var}_{\bar{X}\sim\mathcal{D}_j}\left[\partial_p R(\bar{X};\theta_j^{t,e}) \right]\right) \\
& \leq  16A_1^2A_0^6 \left(\frac{1}{2V}+\frac{|\mathcal{D}_j|/B-1}{|\mathcal{D}_j|-1} \right)
\end{aligned}
\end{gather}
where the last inequality uses Lemma \ref{lem:cov}, the fact $\mathsf{Var}[X]\leq \mathsf{E}\norm{X}^2$ and sampling with and without replacement.

For the term $T_5$, we have
\begin{gather}\label{eq:E}
\begin{aligned}
T_5 &= Tr\left\{\partial_pR_j(\theta_j^{t,e})\left((1-q_j)\tilde{R}^t_{-j}- \sum_{j'\neq j}q_{j'}R_{j'}(\theta_j^{t,e})\right) \right\}^2 \\
&= \sum_p Tr\left\{\partial_pR_j(\theta_j^{t,e})\sum_{j'\neq j}q_{j'}\left(\tilde{R}^t_{j'}- R_{j'}(\theta_j^{t,e})\right) \right\}^2 \\
&\leq \sum_p\norm{\partial_pR_j(\theta_j^{t,e})}_F^2\norm{\sum_{j'\neq j}q_{j'}\left(\tilde{R}^t_{j'}- R_{j'}(\theta_j^{t,e})\right)}_F^2\\
&\leq 4A_1^2A_0^2\norm{\sum_{j'\neq j}q_{j'}\left(\tilde{R}^t_{j'}- R_{j'}(\theta_j^{t,e})\right)}_F^2\\
&\leq 4(1-q_j)A_1^2A_0^2\sum_{j'\neq j}q_{j'}\norm{\tilde{R}^t_{j'}- R_{j'}(\theta_j^{t,e})}_F^2\\
&=4(1-q_j)A_1^2A_0^2\sum_{j'\neq j}q_{j'}\norm{\tilde{R}^t_{j'}-R_{j'}(\theta^{t})+R_{j'}(\theta^{t})- R_{j'}(\theta_j^{t,e})}_F^2\\
&\leq 8(1-q_j)A_1^2A_0^2\sum_{j'\neq j}q_{j'}\left(\norm{\tilde{R}^t_{j'}-R_{j'}(\theta^{t})}_F^2 + \norm{R_{j'}(\theta^{t})- R_{j'}(\theta_j^{t,e})}_F^2\right)\\
\end{aligned}    
\end{gather}
where the second inequality uses Lemma \ref{lem:cov}, and the third inequality uses Jensen's inequality. 

Use Lemma \ref{lem:cov} and mean-value theorem, we have
\begin{gather}\label{eq:temp1}
\norm{R_{j'}(\theta^{t})- R_{j'}(\theta_j^{t,e})}_F^2 \leq 4A_1^2A_0^2\norm{\theta^{t}-\theta_j^{t,e}}_2^2.    
\end{gather}
Denote $D_{j}^{t}=\frac{1}{1-q_j}\sum_{j'\neq j}q_{j'}\norm{\tilde{R}^t_{j'}-R_{j'}(\theta^{t})}_F^2$, we have 
\begin{gather}\label{eq:T5}
\begin{aligned}
T_5 &\leq 8(1-q_j)^2A_1^2A_0^2\left(4A_1^2A_0^2\norm{\theta^{t}-\theta_j^{t,e}}^2 + D_{j}^{t} \right) \\
\end{aligned}    
\end{gather}
Notice that 
\begin{gather}\label{eq:temp2}
\begin{aligned}
\sum_jq_j\norm{\theta^{t,e}-\theta_j^{t,e}}^2 &= \sum_jq_j\norm{\theta^{t,e}-\theta^{t}+\theta^{t}-\theta_j^{t,e}}^2  \\
&\leq \sum_jq_j\norm{\theta^{t}-\theta_j^{t,e}}^2 
\end{aligned}
\end{gather}
then, substitute eq. (\ref{eq:T1}) with eq. (\ref{eq:T3}), eq. (\ref{eq:T4}) and eq. (\ref{eq:T5}), we have
\begin{gather}\label{eq:termA}
\begin{aligned}
T_1 &\leq 3\Biggl(\beta^2\sum_jq_j\norm{\theta^{t,e}-\theta_j^{t,e}}^2 + 16A_1^2A_0^6\sum_jq_j^3\left(\frac{1}{2V} + \frac{|\mathcal{D}_j|/B-1}{|\mathcal{D}_j|-1} \right) \\
&+ 8A_1^2A_0^2\sum_jq_j(1-q_j)^2\left(4A_1^2A_0^2\norm{\theta^{t}-\theta_j^{t,e}}^2 + D_{j}^{t}\right) \Biggr) \\
&\leq 3\Biggl(\left(\beta^2+32A_1^6A_0^6\right)\sum_jq_j\norm{\theta^{t}-\theta_j^{t,e}}^2 + 16A_1^2A_0^6\left(\frac{1}{2V} + \max_j\frac{|\mathcal{D}_j|/B-1}{|\mathcal{D}_j|-1} \right)+ 8A_1^2A_0^2\sum_jq_jD_{j}^{t}\Biggr) \\ 
\end{aligned}   
\end{gather}

\subsubsection{Bounding the term $T_2$}
\begin{gather}\label{eq:T2}
\begin{aligned}
T_2 = \mathsf{Var}\left[v^{t,e}|\mathcal{F}^{t,e} \right]= \mathsf{Var}\left[\sum_jq_jv_j^{t,e}|\mathcal{F}^{t,e} \right] = \sum_jq_j^2 \mathsf{Var}\left[v_j^{t,e}|\mathcal{F}^{t,e} \right]
\end{aligned}    
\end{gather}
Compare eq. (\ref{eq:vp}) and eq. (\ref{eq:vpbar}), we have
\begin{gather}
\begin{aligned}
v_j^{t,e}[p]-\bar{v}_j^{t,e}[p] &= -\partial_pTr\left\{\hat{R}_j^+(\{X_v\}_v,\theta^{t,e}_j)\right\}+\partial_pTr\left\{R_j^+(\theta^{t,e}_j)\right\} \\
&+q_jTr\left\{\hat{R}_j(\{X_v\}_v,\theta^{t,e}_j)\partial_p\hat{R}_j(\{X_v\}_v,\theta^{t,e}_j)\right\} - q_j\mathsf{E}_{\{X_v\}_v}\left[Tr\left\{\hat{R}_j(\{X_v\}_v,\theta^{t,e}_j)\partial_p\hat{R}_j(\{X_v\}_v,\theta^{t,e}_j)\right\}\right] \\
&+(1-q_j)Tr\left\{\partial_p\hat{R}_j(\{X_v\}_v,\theta^{t,e}_j))\tilde{R}^t_{-j} \right\} - (1-q_j)Tr\left\{\partial_pR_j(\theta^{t,e}_j)\tilde{R}^t_{-j} \right\}
\end{aligned}
\end{gather}
\begin{gather}\label{eq:varv}
\begin{aligned}
&\mathsf{Var}\left[v_j^{t,e}|\mathcal{F}^{t,e} \right] = \mathsf{E}_{\{X_v\}_v}\left[\norm{v_j^{t,e}-\bar{v}_j^{t,e}}^2\right]\\
&\leq  3\underbrace{\mathsf{E}_{\{X_v\}_v}\sum_p\left(\partial_pTr\left\{\hat{R}_j^+(\{X_v\}_v,\theta^{t,e}_j)\right\}-\partial_pTr\left\{R_j^+(\theta^{t,e}_j)\right\}\right)^2}_{T_6} \\
&+3q_j^2 \underbrace{\mathsf{E}_{\{X_v\}_v}\sum_p\left(Tr\left\{\hat{R}_j(\{X_v\}_v,\theta^{t,e}_j)\partial_p\hat{R}_j(\{X_v\}_v,\theta^{t,e}_j)\right\} - \mathsf{E}_{\{X_v\}_v}\left[Tr\left\{\hat{R}_j(\{X_v\}_v,\theta^{t,e}_j)\partial_p\hat{R}_j(\{X_v\}_v,\theta^{t,e}_j)\right\}\right]\right)^2}_{T_7} \\
&+3(1-q_j)^2\underbrace{\mathsf{E}_{\{X_v\}_v}\sum_p\left(Tr\left\{\partial_p\hat{R}_j(\{X_v\}_v,\theta^{t,e}_j))\tilde{R}^t_{-j} \right\} - Tr\left\{\partial_pR_j(\theta^{t,e}_j)\tilde{R}^t_{-j} \right\}\right)^2}_{T_8}
\end{aligned}
\end{gather}
For term $T_6$, we have
\begin{gather}\label{eq:F}
\begin{aligned}
T_6&\leq \mathsf{E}_{\bar{X}\sim \mathcal{D}_j}\mathsf{E}_{\{X_v\}_v|\bar{X}}\sum_p \Bigl(\partial_pTr\left\{\hat{R}_j^+(\{X_v\}_v,\theta^{t,e}_j)\right\}-\partial_pTr\left\{R^+(\bar{X},\theta_j^{t,e})\right\}\\
&+\partial_pTr\left\{R^+(\bar{X},\theta_j^{t,e})\right\}-\partial_pTr\left\{R_j^+(\theta_j^{t,e})\right\} \Bigr)^2 \\
&= \mathsf{E}_{\bar{X}\sim \mathcal{D}_j}\mathsf{E}_{\{X_v\}_v|\bar{X}}\sum_p \left(\partial_pTr\left\{\hat{R}_j^+(\{X_v\}_v,\theta^{t,e}_j)\right\}-\partial_pTr\left\{R^+(\bar{X},\theta_j^{t,e})\right\}\right)^2\\
&+\mathsf{E}_{\bar{X}\sim \mathcal{D}_j}\left(\partial_pTr\left\{R^+(\bar{X},\theta_j^{t,e})\right\}-\partial_pTr\left\{R_j^+(\theta_j^{t,e})\right\}\right)^2 \\
&= \sum_p \left(\mathsf{E}_{\bar{X}\sim \mathcal{D}_j}\mathsf{Var}_{\{X_v\}_v|\bar{X}}\left[\partial_pTr\left\{\hat{R}_j^+(\{X_v\}_v,\theta^{t,e}_j)\right\}\right] +\mathsf{Var}_{\bar{X}\sim \mathcal{D}_j}\left[\partial_pTr\left\{R^+(\bar{X},\theta_j^{t,e})\right\}\right] \right) \\
&\leq 4A_1^2A_0^2\left(\frac{1}{V}+\frac{|\mathcal{D}_j|/B-1}{|\mathcal{D}_j|-1}\right)
\end{aligned}    
\end{gather}
where the first equality uses $\mathsf{E}\norm{X-\mathsf{E}[X]}^2\leq \mathsf{E}\norm{X}^2$; the last inequality uses the variance under sampling without replacement.

For the term $T_7$, we have
\begin{gather}\label{eq:Gpre}
\begin{aligned}
T_7 &= \sum_p\mathsf{Var}_{\{X_v\}_v}\left[  Tr\left\{\hat{R}_j(\{X_v\}_v,\theta^{t,e}_j)\partial_p\hat{R}_j(\{X_v\}_v,\theta^{t,e}_j)\right\}\right] \\
&\leq 2\norm{R_j(\theta_j^{t,e})}^2_F\sum_p\mathsf{Var}_{\{X_v\}_v}\left[\partial_p\hat{R}_j(\{X_v\}_v,\theta^{t,e}_j)\right] +2\sum_p\norm{\partial_p R_j(\theta_j^{t,e})}_F^2 \mathsf{Var}_{\{X_v\}_v}\left[\hat{R}_j(\{X_v\}_v,\theta^{t,e}_j)\right] \\
&\leq 2A_0^4\sum_p\mathsf{Var}_{\{X_v\}_v}\left[\partial_p\hat{R}_j(\{X_v\}_v,\theta^{t,e}_j)\right]+8A_1^2A_0^2\mathsf{Var}_{\{X_v\}_v}\left[\hat{R}_j(\{X_v\}_v,\theta^{t,e}_j)\right]
\end{aligned}    
\end{gather}
Notice that
\begin{gather}\label{eq:tg1}
\begin{aligned}
&\sum_p\mathsf{Var}_{\{X_v\}_v}\left[\partial_p\hat{R}_j(\{X_v\}_v,\theta^{t,e}_j)\right]\\
& = \sum_p\mathsf{E}_{\bar{X}\sim \mathcal{D}_j}\mathsf{Var}_{\{X_v\}_v|\bar{X}}\left[\partial_p\hat{R}_j(\{X_v\}_v,\theta^{t,e}_j)\right] + \sum_p\mathsf{Var}_{\bar{X}\sim\mathcal{D}_j}\left[\partial_p R(\bar{X};\theta_j^{t,e}) \right] \\
&\leq 4A_1^2A_0^2\left(\frac{1}{2V} + \frac{|\mathcal{D}_j|/B-1}{|\mathcal{D}_j|-1}\right) 
\end{aligned}
\end{gather}
where the second line uses $\mathsf{Var}[Y]=\mathsf{E}_{X}\mathsf{Var}[Y|X] + \mathsf{Var}[\mathsf{E}[Y|X]]$; the third line uses $\mathsf{Var}[X]\leq \mathsf{E}\norm{X}_F^2$, Lemma \ref{lem:cov}, and sampling with and without replacement. Similarly, we have
\begin{gather}\label{eq:tg2}
\begin{aligned}
&\mathsf{Var}_{\{X_v\}_v}\left[\hat{R}_j(\{X_v\}_v,\theta^{t,e}_j)\right] \\
& = \mathsf{E}_{\bar{X}\sim \mathcal{D}_j}\mathsf{Var}_{\{X_v\}_v|\bar{X}}\left[\hat{R}_j(\{X_v\}_v,\theta^{t,e}_j)\right] + \mathsf{Var}_{\bar{X}\sim\mathcal{D}_j}\left[ R(\bar{X};\theta_j^{t,e}) \right] \\
&\leq A_0^4\left(\frac{1}{2V} + \frac{|\mathcal{D}_j|/B-1}{|\mathcal{D}_j|-1}\right) 
\end{aligned}
\end{gather}
Plug eq. (\ref{eq:tg1}) and eq. (\ref{eq:tg2}) into eq. (\ref{eq:Gpre}), we have
\begin{gather}\label{eq:G}
T_7 \leq 8A_1^2A_0^6\left(\frac{1}{2V} + \frac{|\mathcal{D}_j|/B-1}{|\mathcal{D}_j|-1}\right)     
\end{gather}
For the term $T_8$, using eq. (\ref{eq:tg1}) we have
\begin{gather}\label{eq:H}
\begin{aligned}
T_8 &\leq \sum_p \mathsf{Var}_{\{X_v\}_v}\left[\partial_p\hat{R}_j(\{X_v\}_v,\theta^{t,e}_j)\right]\norm{\tilde{R}^t_{-j}}^2_F\\
&\leq 4A_1^2A_0^2\left(\frac{1}{2V} + \frac{|\mathcal{D}_j|/B-1}{|\mathcal{D}_j|-1}\right)\norm{\tilde{R}^t_{-j}}^2_F
\end{aligned}
\end{gather}

Substitute eq. (\ref{eq:T2}) with eq. (\ref{eq:varv}), eq. (\ref{eq:F}), eq. (\ref{eq:G}) and eq. (\ref{eq:H}), we have
\begin{gather}
\begin{aligned}
T_2 &\leq 12A_1^2A_0^2\sum_jq_j^2\left(\frac{1}{V}+\frac{|\mathcal{D}_j|/B-1} {|\mathcal{D}_j|-1}\right) + 24A_1^2A_0^6\sum_jq_j^4\left(\frac{1}{V} + \frac{|\mathcal{D}_j|/B-1}{|\mathcal{D}_j|-1}\right) \\
&+ 12A_1^2A_0^2\sum_jq_j^2(1-q_j)^2\left(\frac{1}{V} + \frac{|\mathcal{D}_j|/B-1}{|\mathcal{D}_j|-1}\right)\norm{\tilde{R}^t_{-j}}^2_F \\
& \leq \left(\frac{1}{V}+\max_j\frac{|\mathcal{D}_j|/B-1} {|\mathcal{D}_j|-1}\right)\left(12A_1^2A_0^2 + 24A_1^2A_0^6 + 12A_1^2A_0^2\sum_jq^2_j(1-q_j)^2\norm{\tilde{R}^t_{-j}}^2_F \right)
\end{aligned}
\end{gather}

\subsubsection{Combine the results}
Take expectation on both sides of eq. (\ref{eq:rec}) conditioned on $\mathcal{F}^{t,0}$, we have
\begin{gather}\label{eq:rect0}
\begin{aligned}
\expc{\mathcal{L}(\theta^{t,e+1})|\mathcal{F}^{t,0}}\leq \expc{\mathcal{L}(\theta^{t,e})|\mathcal{F}^{t,0}} - \frac{\eta}{2} \expc{\norm{\nabla \mathcal{L}(\theta^{t,e})}^2|\mathcal{F}^{t,0}}+\frac{\eta}{2}\expc{T_1|\mathcal{F}^{t,0}} + \expc{T_2|\mathcal{F}^{t,0}}.
\end{aligned}
\end{gather}
Notice that 
\begin{gather}\label{eq:temp3}
\begin{aligned}
\expc{\norm{\theta^{t}-\theta_j^{t,e}}^2|\mathcal{F}^{t,0}} &= \expc{\norm{\sum_{e'=0}^{e-1}\eta
v_j^{t,e'}}^2_2|\mathcal{F}^{t,0}} \\
& \leq E\sum_{e'=0}^{e-1}\eta^2\expc{\norm{v_j^{t,e}}_2^2|\mathcal{F}^{t,0}} \\
& \leq E^2\eta^2(8A_1^2A_0^2 + 8A_1^2A_0^2(A_0^4 + H^2\sigma^2) )
\end{aligned}    
\end{gather} 
where the last line uses Lemma \ref{lem:bdgd}. Also notice that
\begin{gather}
\begin{aligned}
\expc{D_{j}^{t}|\mathcal{F}^{t,0}} &= \frac{1}{1-q_j}\sum_{j'\neq j}q_{j'}\expc{\norm{\tilde{R}^t_{j'}-R_{j'}(\theta^{t})}_F^2|\mathcal{F}^{t,0}}\\
&= \frac{1}{1-q_j}\sum_{j'\neq j}q_{j'}\left[\expc{\norm{\tilde{R}^t_{j'}-\check{R}_{j'}(\theta^{t})}_F^2|\mathcal{F}^{t,0}}+\expc{\norm{\check{R}_{j'}-R_{j'}(\theta^{t})}_F^2|\mathcal{F}^{t,0}}\right]\\
&=H^2\sigma^2 + \frac{1}{2V}A_0^4
\end{aligned}    
\end{gather} 
where $\check{R}_{j'}$ is the empirical correlation matrix before applying DP noise. Then we have
\begin{gather}\label{eq:termAexp}
\begin{aligned}
\expc{T_1|\mathcal{F}^{t,0}}&\leq3\Biggl(\left(\beta^2+32A_1^6A_0^6\right) E^2\eta^2\left(8A_1^2A_0^2 + 8A_1^2A_0^2\left(A_0^4 + H^2\sigma^2\right) \right) \\
&+ 16A_1^2A_0^6\left(\frac{1}{V} + \max_j\frac{|\mathcal{D}_j|/B-1}{|\mathcal{D}_j|-1} \right)+ 8A_1^2A_0^2\left(H^2\sigma^2 + \frac{1}{2V}A_0^4\right)\Biggr).
\end{aligned}   
\end{gather}

For the term $\expc{T_2|\mathcal{F}^{t,0}}$, we have
\begin{gather}
\begin{aligned}
\expc{\norm{\tilde{R}^t_{-j}}^2_F|\mathcal{F}^{t,0}} &= \expc{\norm{\bar{R}^t_{-j}}^2_F|\mathcal{F}^{t,0}} + \expc{\norm{n^t_{-j}}^2_F|\mathcal{F}^{t,0}} \\
& = A_0^4 + \frac{1}{(1-q_j)^2}\sum_{j'\neq j}q_j'^2H^2\sigma^2
\end{aligned}    
\end{gather}
\begin{gather}\label{eq:T2exp}
\begin{aligned}
\expc{T_2|\mathcal{F}^{t,0}} \leq \left(\frac{1}{V}+\max_j\frac{|\mathcal{D}_j|/B-1} {|\mathcal{D}_j|-1}\right)\left(12A_1^2A_0^2 + 26A_1^2A_0^6 + 12A_1^2A_0^2H^2\sigma^2 \right) 
\end{aligned}
\end{gather}
where we use fact $q_j(1-q_j)^2\leq \frac{4}{27}$. Combine eq. (\ref{eq:rect0}), eq. (\ref{eq:termAexp}) and eq. (\ref{eq:T2exp}), we have
\begin{gather}\label{eq:rectcomb}
\begin{aligned}
&\expc{\mathcal{L}(\theta^{t,e+1})|\mathcal{F}^{t,0}}-\expc{\mathcal{L}(\theta^{t,e})|\mathcal{F}^{t,0}} \\
&\leq - \frac{\eta}{2} \expc{\norm{\nabla \mathcal{L}(\theta^{t,e})}^2|\mathcal{F}^{t,0}} + \frac{\eta^3}{2} C_1E^2(H^2\sigma^2+C_2) + \frac{\eta^2}{2}C_3\left(\frac{1}{V}+\max_j\frac{|\mathcal{D}_j|/B-1} {|\mathcal{D}_j|-1}\right)\left(H^2\sigma^2 + C_4\right) \\
&+\frac{\eta}{2}C_5\left(\frac{1}{V}+\max_j\frac{|\mathcal{D}_j|/B-1} {|\mathcal{D}_j|-1}\right)+\frac{\eta}{2}C_6\left(H^2\sigma^2 + \frac{1}{V}C_7\right)
\end{aligned}
\end{gather}
where $C_1,C_2,...,C_7$ are constant depending $A_0$, $A_1$ and $A_2$. Telescope eq. (\ref{eq:rectcomb}) and take expectation, we have
\begin{gather}
\begin{aligned}
&\frac{1}{TE}\sum_{t=0}^{T-1}\sum_{e=0}^{E-1}\expc{\norm{\nabla \mathcal{L}(\theta^{t,e})}^2} \\
&\leq \frac{2}{\eta TE}\left(\expc{\mathcal{L}(\theta^{t,0})} -\expc{\mathcal{L}(\theta^{t,E})}\right) + \eta^2C_1E^2(H^2\sigma^2+C_2)+\eta C_3\left(\frac{1}{V}+\max_j\frac{|\mathcal{D}_j|/B-1} {|\mathcal{D}_j|-1}\right)\left(H^2\sigma^2 + C_4\right) \\
&+C_5\left(\frac{1}{V}+\max_j\frac{|\mathcal{D}_j|/B-1} {|\mathcal{D}_j|-1}\right)+C_6\left(H^2\sigma^2 +\frac{1}{V} C_7\right).
\end{aligned}    
\end{gather}
Recall the choice $\eta = \mathcal{O}\left(\frac{1}{\sqrt{TE}}\right)$, we have
\begin{gather}
\begin{aligned}
&\frac{1}{TE}\sum_{t=0}^{T-1}\sum_{e=0}^{E-1}\expc{\norm{\nabla \mathcal{L}(\theta^{t,e})}^2}\\
&\leq \mathcal{O}\left(\frac{E^2(H^2\sigma^2+C_2)}{TE} + \frac{\mathcal{L}(\theta^{0})+\left(\frac{1}{V}+\max_j\frac{|\mathcal{D}_j|/B-1} {|\mathcal{D}_j|-1}\right)\left(H^2\sigma^2 + C_4\right)}{\sqrt{TE}} + \frac{1}{V}+\max_j\frac{|\mathcal{D}_j|/B-1} {|\mathcal{D}_j|-1}+H^2\sigma^2\right)   
\end{aligned}
\end{gather}

\subsection{Partial Participation Case}
Partial participation results in perturbation in aggregation.

\begin{gather}
\begin{aligned}
&\mathsf{E}\left[\norm{\nabla \mathcal{L}(\theta^{t,E}_{\mathcal{J}^t})}^2-\norm{\nabla \mathcal{L}(\theta^{t,E})}^2|
\mathcal{J}_t\right] \\
&= 2\mathsf{E}_{\mathcal{J}_t}\left[\nabla \mathcal{L}^T(\theta^{t,E}+\lambda h)\nabla^2\mathcal{L}(\theta^{t,E}+\lambda h,\mathcal{D})h \right] \\
&\leq 2W\beta \mathsf{E}_{\mathcal{J}_t}\left[\norm{\theta^{t,E}- \theta^{t,E}_{\mathcal{J}^t}}\right] \\
&\leq 2W\beta \sqrt{\mathsf{E}_{\mathcal{J}^t}\left[\norm{\theta^{t,E}- \theta^{t,E}_{\mathcal{J}^t}}_2^2\right]} \\
&\leq 2\eta\beta EW^2\sqrt{\frac{J/|\mathcal{J}^t|-1}{J-1}}
\end{aligned}
\end{gather}
where one can easily verify that $W = 8A_1^2A_0^2 + 8A_1^2A_0^6$ from Lemma \ref{lem:bdgd} servers a bound for gradient norm; the second line uses mean-value theorem. Another aspect is that $\tilde{R}_j$ is less frequently updated on sever. Therefore the term $\expc{D_{j}^{t}|\mathcal{F}^{t,0}}$ should involve an additional term accounting to aging of correlation matrix,
\begin{gather}
\begin{aligned}
\expc{D_{j}^{t}|\mathcal{F}^{t,0}} \leq H^2\sigma^2 + \frac{1}{2V}A_0^4 + C_8\eta^2.
\end{aligned}    
\end{gather} 
The reason is that the difference between the current and old correlation matrix is proportional to the distance between the current and old variables (shown in eq. (\ref{eq:temp1})), which is proportional to $E^2\eta^2$ (shown in eq. (\ref{eq:temp3})). Thus we finally have 
\begin{gather}
\begin{aligned}
\frac{1}{TE}\sum_{t=0}^{T-1}\sum_{e=0}^{E-1}\expc{\norm{\nabla \mathcal{L}(\theta^{t,e})}^2}&\leq \mathcal{O}\Biggl(\frac{E^2(H^2\sigma^2+C_2)}{TE} \\
&+ \frac{\mathcal{L}(\theta^{0})+\left(\frac{1}{V}+\max_j\frac{|\mathcal{D}_j|/B-1} {|\mathcal{D}_j|-1}\right)\left(H^2\sigma^2 + C_4\right)+ E\sqrt{\frac{J/|\mathcal{J}^t|-1}{J-1}}}{\sqrt{TE}} \\
&+ \frac{1}{V}+\max_j\frac{|\mathcal{D}_j|/B-1} {|\mathcal{D}_j|-1}+H^2\sigma^2\Biggr)   
\end{aligned}
\end{gather}

\begin{table} [htbp]
\caption{Implementation details of FedSC with DP protection: full client participation }
\begin{center}
\begin{footnotesize}
\begin{tabular}{c|c|c|c|c}
\hline
        &  $\mu$ & $\sigma $  & round indices & local dataset size  \\
\hline
SVHN $(\epsilon =3, \delta=10^{-2})$ & $2$ & $0.0034$ & $t>100$ & $10,000$ \\
\hline
SVHN $(\epsilon =6, \delta=10^{-2})$ & $2$ & $0.0018$ & $t>100$ & $10,000$ \\
\hline
SVHN $(\epsilon =3, \delta=10^{-4})$ & $2$ & $0.0048$ & $t>100$ & $10,000$ \\
\hline
SVHN $(\epsilon =8, \delta=10^{-4})$ & $2$ & $0.0018$ & $t>100$ & $10,000$ \\
\hline
CIFAR10 $(\epsilon =3, \delta=10^{-2})$ & $4$ & $0.01$ & $t>150$ & $5,000$  \\
\hline
CIFAR10 $(\epsilon =6, \delta=10^{-2})$ & $4$ & $0.0052$ & $t>100, t\%2 =0$ & $5,000$  \\
\hline
CIFAR10 $(\epsilon =3, \delta=10^{-4})$ & $4$ & $0.012$ & $t>150$ & $5,000$  \\
\hline
CIFAR10 $(\epsilon =8, \delta=10^{-4})$ & $4$ & $0.0051$ & $t>100, t\%2 =0$ & $5,000$  \\
\hline
CIFAR100 $(\epsilon =6, \delta=10^{-2})$ & $5$ & $0.013$ & $t>100, t\%2 =0$ & $2,500$  \\
\hline
CIFAR100$(\epsilon =12, \delta=10^{-2})$ & $5$ & $0.0075$ & $t>100, t\%2 =0$ & $2,500$  \\
\hline
CIFAR100 $(\epsilon =3, \delta=10^{-4})$ & $5$ & $0.04$ & $t>100, t\%2 =0$ & $2,500$  \\
\hline
CIFAR100$(\epsilon =8, \delta=10^{-4})$ & $5$ & $0.013$ & $t>100, t\%2 =0$ & $2,500$  \\
\hline
\end{tabular}
\end{footnotesize}
\end{center}
\label{tab:compn}
\end{table}

\section{Detailed Implementation of FedSC}
Recall the local objective
\begin{gather}
\begin{aligned}
\mathcal{L}^{SC}_j(\theta; \bar{R}_{-j})&=-Tr\{R^+_j(\theta)\} + \frac{1}{2}\alpha_j\norm{R_j(\theta)}_F^2 + (1-\alpha_j)Tr\{R_j(\theta)\bar{R}_{-j}\}
\end{aligned}   
\end{gather}
here we replace $q_j$ with a general coefficient $\alpha_j$, and decay it linearly from $1$ to $0.2$ along with communication round indices. The behind motivation is as follows. At the beginning of the training, moving direction from the global objective and the average local objective tend to align closely. Moreover, the correlation matrices of clients are not yet stable at this stage, making it less critical to at early stages. Therefore, we choose large $\alpha_j$ for quicker start. Conversely, correlation matrices converges and becomes stable at the end of training, thus we give the inter-client contrast larger weights, i.e., smaller $\alpha_j$.

We also make modifications when DP protection is applied. Based on the above analysis, we start sharing at the middle or late stages of the training to save privacy budgets. Following are the detailed implementation details. For partial client participation, we only change $\sigma$ according to the ratio of participation.



\end{document}